\documentclass{article}

\usepackage{microtype}
\usepackage{graphicx}
\usepackage{subfigure}
\usepackage{booktabs} % for professional tables
\usepackage{enumitem} 

\usepackage{hyperref}

\usepackage{amsthm} %for macros
\usepackage{amssymb} %for macros
\usepackage{mathtools}

\newcommand{\Rtypmax}{R_{\mathrm{typ-max}}}
\newcommand{\Rmax}{R_{\mathrm{max}}}
\newcommand{\diag}{\mathrm{diag}}

 \usepackage{relsize}

\usepackage{etoolbox}
\usepackage{xparse}

\newcommand{\R}{\Rbb}

\renewcommand{\vec}[1]{\ensuremath{\mathbf{#1}}}
\newcommand{\vecs}[1]{\ensuremath{\mathbf{\boldsymbol{#1}}}}
\newcommand{\mat}[1]{\ensuremath{\mathbf{#1}}}
\newcommand{\mats}[1]{\ensuremath{\mathbf{\boldsymbol{#1}}}}
\newcommand{\ten}[1]{\mat{\ensuremath{\boldsymbol{\mathcal{#1}}}}}

\usepackage{pgffor}

\foreach \x in {A,...,Z}{%
\expandafter\xdef\csname \x bb\endcsname{\noexpand\ensuremath{\noexpand\mathbb{\x}}}
}

\foreach \x in {A,...,Z}{%
\expandafter\xdef\csname \x cal\endcsname{\noexpand\ensuremath{\noexpand\mathcal{\x}}}
}

\foreach \x in {A,...,Z}{%
\expandafter\xdef\csname \x t\endcsname{\noexpand\ensuremath{\noexpand\ten{\x}}}
}

\foreach \x in {A,...,Z}{%
\expandafter\xdef\csname \x b\endcsname{\noexpand\ensuremath{\noexpand\mat{\x}}}
}

\foreach \x in {a,...,r}{%
\expandafter\xdef\csname \x b\endcsname{\noexpand\ensuremath{\noexpand\vec{\x}}}
}
\foreach \x in {t,...,z}{%
\expandafter\xdef\csname \x b\endcsname{\noexpand\ensuremath{\noexpand\vec{\x}}}
}

\newcommand{\A}{\mat{A}}
\newcommand{\B}{\mat{B}}
\newcommand{\C}{\mat{C}}

\newcommand{\T}{\ten{T}}

\newcommand{\U}{\mat{U}}

\newcommand{\I}{\mat{I}}

\newcommand{\M}{\mat{M}}

\newcommand{\x}{\vec{x}}
\newtheorem*{theorem*}{Theorem}
\newtheorem*{corollary*}{Corollary}%
\newtheorem*{proposition*}{Proposition}%
\ifcsdef{theorem}{}{
\newtheorem{theorem}{Theorem}%
\newtheorem{lemma}{Lemma}%
\newtheorem{definition}{Definition}%
\newtheorem{corollary}[theorem]{Corollary}%
}
\newtheorem*{pbm*}{Problem}%
\newtheorem*{algo*}{Algorithm}%

\newtheorem{manualnumberedlemma}{Lemma}
\newenvironment{manuallemma}[1]{%
  \renewcommand\themanualnumberedlemma{#1}%
  \manualnumberedlemma
}{\endmanualnumberedlemma}

\newtheorem{manualnumberedtheorem}{Theorem}
\newenvironment{manualthm}[1]{%
  \renewcommand\themanualnumberedtheorem{#1}%
  \manualnumberedtheorem
}{\endmanualnumberedtheorem}

\newtheorem{manualnumberedcorollary}{Corollary}
\newenvironment{manualcor}[1]{%
  \renewcommand\themanualnumberedcorollary{#1}%
  \manualnumberedcorollary
}{\endmanualnumberedcorollary}

\DeclareMathOperator*{\rank}{rank}

\newcommand{\nstates}{n}

\newcommand{\szerosymbol}{\alpha}
\newcommand{\szero}{\vecs{\szerosymbol}}

\newcommand{\sinfsymbol}{\omega}
\newcommand{\sinf}{\vecs{\sinfsymbol}}

\DeclareDocumentCommand{\wa}{  O{A} O{\szero} O{\sinf} }%
{(#2,\{\mat{#1}^\sigma\}_{\sigma\in\Sigma},#3)}
\DeclareDocumentCommand{\waR}{  O{A} O{\Rbb^\nstates} O{\szero} O{\sinf} }%
{(#2,#3,\{\mat{#1}^\sigma\}_{\sigma\in\Sigma},#4)}

\newcommand{\vvsinfsymbol}{\Omega}
\newcommand{\vvsinf}{\mats{\vvsinfsymbol}}
\DeclareDocumentCommand{\vvwa}{  O{A} O{\szero} O{\vvsinf} }%
{(#2,\{\mat{#1}^\sigma\}_{\sigma\in\Sigma},#3)}

\newcommand{\tzerosymbol}{\alpha}
\newcommand{\tzero}{\vecs{\tzerosymbol}}

\newcommand{\tinfsymbol}{\omega}
\newcommand{\tinf}{\vecs{\tinfsymbol}}

\DeclareDocumentCommand{\wta}{ O{T} O{\Rbb^\nstates} O{\tzero} O{\tinf} O{\Fcal}}%
{(#2,#3,\{\ten{#1}^g\}_{g\in #5_{\geq 1}},\{#4^\sigma\}_{\sigma\in #5_0})}

\DeclareDocumentCommand{\trees}{g}{\IfNoValueTF{#1}{\mathfrak{T}}{\mathfrak{T}_{#1}}}
\DeclareDocumentCommand{\contexts}{g}{\IfNoValueTF{#1}{\mathfrak{C}}{\mathfrak{C}_{#1}}}
\newcommand{\gwmprod}{\diamond}

\DeclareDocumentCommand{\gwm}{  O{M} O{\Fbb^\nstates}}{(#2, \{\ten{#1}^x\}_{x\in\Sigma})}
\DeclareDocumentCommand{\gwmcirc}{  O{M} O{\Rbb^\nstates}}{(#2, \{\mat{#1}^\sigma\}_{\sigma\in\Sigma})}
\DeclareDocumentCommand{\dgwm}{ O{M} O{\Fbb^\nstates}}{(#2, \{\ten{#1}^x\}_{x\in\Sigma},\gwmprod)}

\newcommand{\cprank}[1]{\mathrm{rank}_{\mathrm{CP}}(#1)}

\newcommand{\hCPRNN}[1]{\mathcal{H}_{\mathrm{CPRNN}}(#1)}

\newcommand{\hRNN}[1]{\mathcal{H}_{2\mathrm{RNN}}(#1)}

\newcommand{\hMIRNN}[1]{\mathcal{H}_{\mathrm{MIRNN}}(#1)}
\newcommand{\hCPBIRNN}[1]{\mathcal{H}_{\mathrm{CPBIRNN}}(#1)}

\newcommand{\linearmaps}[1]{\mathcal{L}^{#1}}
\newcommand{\CP}[1]{ [\![#1]\!]}

\usepackage{pgffor}
\foreach \x in {A,...,Z}{%
\expandafter\xdef\csname \x bb\endcsname{\noexpand\ensuremath{\noexpand\mathbb{\x}}}
}

\foreach \x in {A,...,Z}{%
\expandafter\xdef\csname \x cal\endcsname{\noexpand\ensuremath{\noexpand\mathcal{\x}}}
}

\usepackage{pgffor}
\foreach \x in {A,...,Z}{%
\expandafter\xdef\csname \x ten\endcsname{\noexpand\ensuremath{\noexpand\ten{\x}}}
}

\foreach \x in {A,...,Z}{%
\expandafter\xdef\csname \x mat\endcsname{\noexpand\ensuremath{\noexpand\mat{\x}}}
}

\foreach \x in {a,...,z}{%
\expandafter\xdef\csname \x vec\endcsname{\noexpand\ensuremath{\noexpand\mat{\x}}}
}

\usepackage[symbol]{footmisc}

\usepackage{appendix}
\usepackage{xcolor}

\usepackage[round]{natbib}

\usepackage[accepted]{icml2024}

\usepackage{amsmath}
\usepackage{amssymb}
\usepackage{mathtools}
\usepackage{amsthm}

\usepackage[capitalize,noabbrev]{cleveref}

\usepackage[textsize=tiny]{todonotes}
\usepackage{soul}

\icmltitlerunning{A Tensor Decomposition Perspective on Second-order RNNs}

\begin{document}

\twocolumn[
\icmltitle{A Tensor Decomposition Perspective on Second-order RNNs}

\icmlsetsymbol{equal}{*}

\begin{icmlauthorlist}
\icmlauthor{Maude Lizaire}{udm}
\icmlauthor{Michael Rizvi-Martel}{udm}
\icmlauthor{Marawan Gamal Abdel Hameed}{udm}
\icmlauthor{Guillaume Rabusseau}{udm,cif}
\end{icmlauthorlist}

\icmlaffiliation{cif}{CIFAR AI Chair}
\icmlaffiliation{udm}{Mila \& DIRO, Université de Montréal, Montreal, Canada}

\icmlcorrespondingauthor{Maude Lizaire}{maude.lizaire@umontreal.ca}
\icmlcorrespondingauthor{Guillaume Rabusseau}{grabus@iro.umontreal.ca}

\icmlkeywords{Machine Learning, ICML, Recurrent Neural Networks, CP Decomposition, Tensor Decomposition, Expressivity, rank}

\vskip 0.3in
]

\printAffiliationsAndNotice{}  

\begin{abstract}
Second-order Recurrent Neural Networks (2RNNs) extend RNNs by leveraging second-order interactions for sequence modelling. These models are provably more expressive than their first-order counterparts and have connections to well-studied models from formal language theory. However, their large parameter tensor makes computations intractable. To circumvent this issue, one approach known as MIRNN consists in limiting the type of interactions used by the model. Another is to leverage tensor decomposition to diminish the parameter count. In this work, we study the model resulting from parameterizing 2RNNs using the CP decomposition, which we call CPRNN. Intuitively, the rank of the decomposition should reduce expressivity. We analyze how rank and hidden size affect model capacity and show the relationships between RNNs, 2RNNs, MIRNNs, and CPRNNs based on these parameters. We support these results empirically with experiments on the Penn Treebank dataset which demonstrate that, with a fixed parameter budget, CPRNNs outperforms RNNs, 2RNNs, and MIRNNs with the right choice of rank and hidden size.
\end{abstract}

\begin{figure}[t]
\includegraphics[scale=0.3]{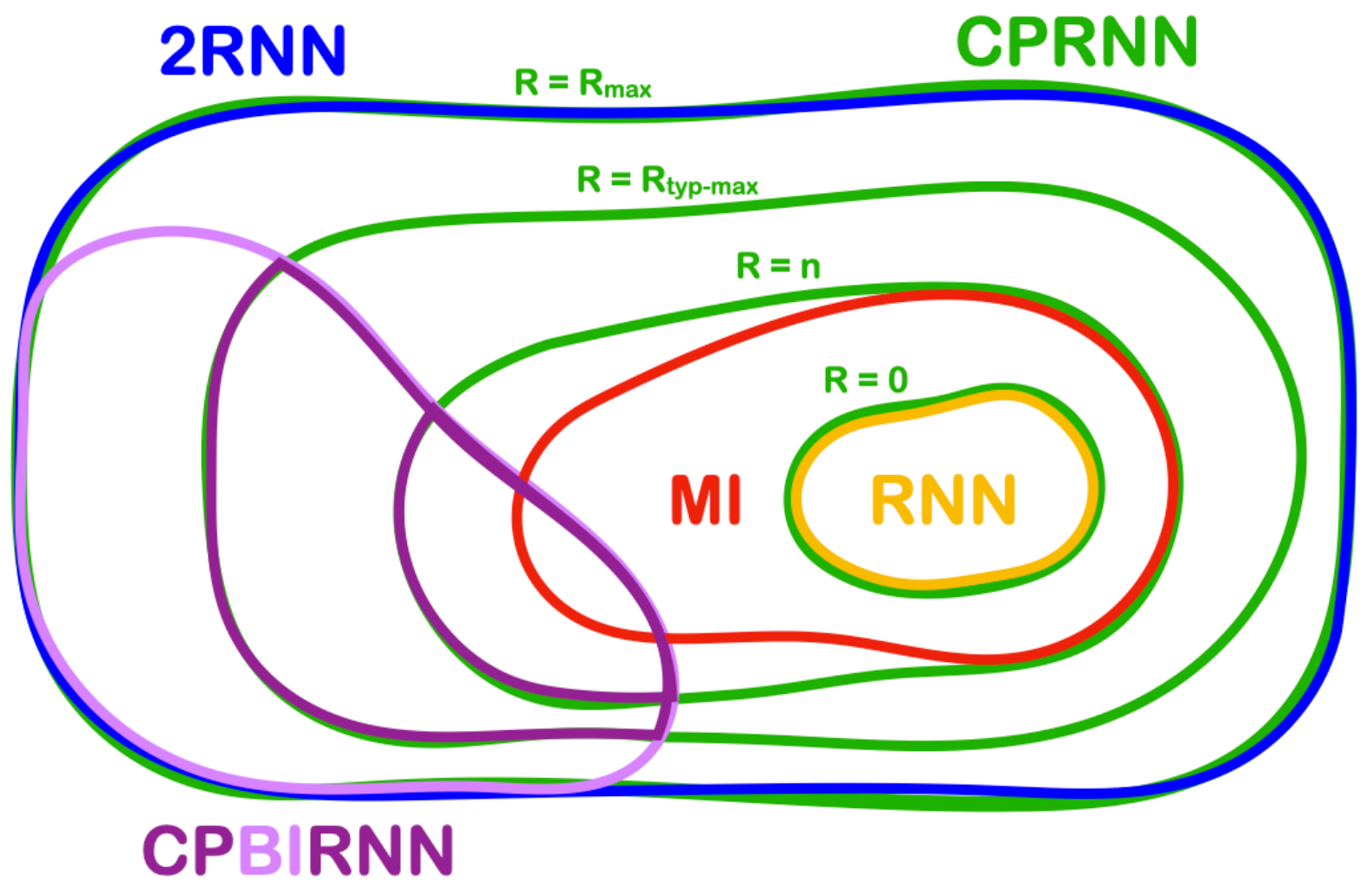}
\caption{Overview of expressivity relationships among Recurrent Neural Network architectures. 2RNNs (blue) encompasses all other models. BIRNN (purple) and RNN (yellow) are a subclasses of 2RNNs as they only have first-order and second-order interactions respectively. MIRNN (red) includes element wise multiplicative interactions. CPRNNs expressive power varies with the rank $R$, reaching the same capacity as 2RNNs when $R=R_{\text{max}}$.}\label{fig:venn_diagram}
\end{figure}

\section{Introduction}
Recurrent neural networks (RNNs) have been pivotal to the deep learning revolution~\cite{elman1990finding, hopfield1982neural}. Given their strong inductive bias, they are a natural choice when it comes to sequence modelling.
Although transformers~\cite{vaswani2017attention} remain the de facto choice for language modelling, there exists many domains in which RNNs still excel. Moreover, with the advent of state space models~\cite{gu2020hippo,gu2021efficiently}, such architectures have seen an increase in popularity~\cite{orvieto2023resurrecting}, especially with regard to long-range dependencies. Second-order Recurrent Neural Networks (2RNNs) are a generalization of RNNs which integrates second-order interactions between hidden states and inputs. These second-order interactions make 2RNNs strictly more expressive than their first-order counterparts. However, the tensors parameterizing the second-order term quickly become very large as the hidden state and input dimensions grow~\citep{Giles1989, goudreau1994first}, thus making inference with such models intractable. 
\begin{figure*}[th]
\begin{center}
\includegraphics[width=1\textwidth]{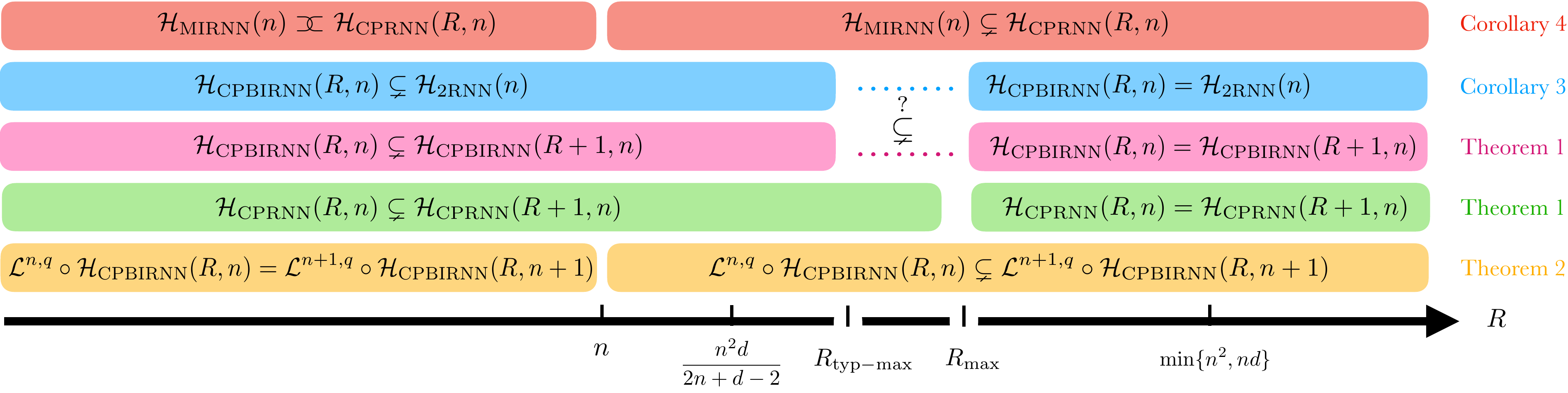}
\end{center}
\vspace{-0.5cm}
\caption{Relations of expressivity between CP(BI)RNNs, 2RNNs and MIRNNs as a function of the rank $R$ of the CP(BI)RNN; $n$ denotes the hidden dimension, $d$ the input dimension and $\Rmax$~(resp. $\Rtypmax$) the maximal CP rank~(resp. maximal typical CP rank).  Theorems~\ref{thm:rankcprnn}\&\ref{thm:hiddencprnn} and Corollaries~\ref{cor:2rnn}\&\ref{thm:mirnn} detail these results.}\label{fig:sumary}
\end{figure*}
A simple but drastic approach to reduce the parameter count is to keep only the second-order interactions from the component-wise product between input and hidden state~\citep{wu2016multiplicative}. The resulting model, Multiplicative Integration RNN (MIRNN), was shown to alleviate the vanishing/exploding gradient problem and to improve performance on language modelling tasks. A linear version of this model without first-order interactions was also introduced under the name \textit{Recurrent Arithmetic Circuits} to analyze the benefits of depth in RNNs~\citep{levine2018benefits}. 

A more sophisticated approach would be to leverage tensor decomposition to compress the third-order tensor parameterizing the 2RNN. While several decompositions could be used to reduce parameter count in 2RNNs, in this work we focus on the popular CP decomposition~\citep{kiers2000towards}. We call the resulting model CPRNN.
\citet{sutskever2011generating} empirically demonstrated the potential of this architecture by successfully deploying it on one of the largest RNN applications. However, there was no theoretical analysis of this new model, a gap we propose to fill with this work.

It is interesting to note that while the model proposed in~\citet{sutskever2011generating} is indeed a CPRNN---that is a 2RNN with second-order weights parameterized by a CP decomposition--- there is no mention of tensor decomposition or of the CP decomposition in their work.

Similarly to matrices, the sets of tensors of CP rank at most $R$, for increasing values of $R$, form a nested family in the space of tensors. Intuitively, an increase in rank strictly increases the capacity up to a point of saturation. This naturally leads to the following questions: 
\begin{itemize}
    \item How does increasing the capacity of the tensor parameter relate to the expressivity of CPRNNs? 
    \item  How does the point of saturation in the tensor parameter space translate in the function space of CPRNNs? In particular, how is it affected by the hidden dimension? 
    \item More generally, how do the rank and the hidden dimension interplay in controlling the capacity of CPRNNs?
\end{itemize}

To the best of our knowledge, our work is the first to formally address these questions. We start from the observation that the rank in CPRNNs acts as a hyper-parameter interpolating between first-order and second-order RNNs. We formalize this observation by showing that CPRNNs are equivalent to 2RNNs above some maximal rank. We also show that MIRNNs are a special case of CPRNNs. More precisely, we show that MIRNNs correspond to a specific point on this interpolation line where the rank and hidden dimension of a CPRNN are both equal. In fact, we show that CPRNNs are strictly more expressive than MIRNNs past this threshold, i.e., when  rank is greater than  hidden size. Tensor decompositions thus prove to be an effective approach to characterize the expressivity of RNNs with different degrees of multiplicative interactions. Figure \ref{fig:venn_diagram} illustrates the hierarchy established by the rank of CPRNNs, where the threshold values between different classes explored in this paper are directly connected to the hidden size.

We investigate to which extent theory holds in practice for models trained on real data by conduction a set of experiments on the Penn Tree Bank dataset\footnote{Code base for this paper can be found at \url{https://github.com/MaudeLiz/cprnn} }. We first show that, as expected, the performance of CPRNNs increases as either the rank or hidden dimension grows, and that the rank of CPRNNs naturally controls the parameter/expressivity tradeoff, interpolating between RNNs and 2RNNs in a fine-grained fashion. More interestingly, our experimental results show that for any model size budget there always exists a choice of hidden size and rank that fits the parameter budget and outperforms all 2RNNs, RNNs and MIRNNs with the same parameter count. Another interesting behavior of CPRNNs illustrated by our experiments is that, since rank and hidden size are tied when the number of parameters is fixed, there are two underfitting regimes: one when the rank is too small~(where the rank acts as the bottleneck) and one when it is too large~(where it is the hidden dimension that is too small and acts as a bottleneck). 

\paragraph{Summary of contributions} We present a formal study of CPRNNs: a natural approach to reducing parameter count in second-order RNNs. While CPRNNs have been empirically considered previously, to the best of our knowledge, this work is the first to thoroughly and rigorously analyze it from a formal perspective. Our analysis establishes several novel and non-trivial results related to the expressivity of CPRNNs, how the rank and hidden size of CPRNN interact, and how CPRNNs encompass previous second-order recurrent models. Beyond this theoretical analysis, we design several experiments on real data showcasing various properties and empirical behaviours of CPRNNs. In particular, we show that CPRNNs  always offer a better tradeoff between expressivity and size than RNNs, 2RNNs and MIRNNs. 

\paragraph{Related work}
Extensive work has been done on the expressive power of RNNs and how they relate to formal languages (see e.g., \citet{siegelmann1994analog,chen2017recurrent,weiss2018practical,korsky2019computational,merrill2019sequential,merrill2020formal,deletang2022neural}). In the case of 2RNNS there is a direct correspondence with finite states machines. Indeed, the linear form of 2RNNs was shown to generalize weighted finite automaton (WFA) to non-discrete inputs \citep{li2022connecting}. From a language modelling perspective, multiplicative interactions draw interest for their potential to represent more complex dependencies such as compositional semantics \citep{Irsoy2014ModelingCW, Jayakumar2020Multiplicative}. This motivated the design of various multiplicative RNNs architectures \citep{tjandra2016gated, krause2017,sutskever2011generating,wu2016multiplicative,su2024language}, among which CPRNNs and MIRNNs stand out for their parameter efficiency. At the same time, tensor decomposition methods have been used to compress sequential models \citep{yang2017tensor, ye2018learning,ma2021cp,wang2021kronecker, tjandra2018tensor} and tackle optimization problems involving multivariate higher-order polynomial functions \citep{yu2017long,ayvaz2022cpd,dubey2022scalable}. More broadly, tensor decompositions have been used to study the expressivity of different neural network architectures, in particular to explore the benefits of depth~\citep{cohen2016convolutional, cohen2016expressive, cohen2016inductive, sharir2017expressive, levine2018benefits, alexander2023makes, razin2024ability} and structural alignment between data and model~\citep{balda2018tensor, khrulkov2017expressive, khrulkov2019generalized}. The analysis in most of these work relies on comparing different tensor decomposition (or tensor network) structures, in contrast, our analysis relies on studying the effect of the rank of a fixed tensor network structure. This strategy of leveraging the rank of a tensor decomposition to characterize the expressive power of models has also  been used for deep polynomial networks in~\citep{kileel2019expressive}.  

Finally, it is worth mentioning that the popular family of state space models~\cite{hamilton1994state,gu2021efficiently} also have second-order counterparts~\cite{sattar2022finite}. These model variants appear mostly in dynamical systems literature and are not quite as well known in the machine learning community.

\section{Preliminaries}
We first introduce notations, tensor decomposition and the various RNN models studied in this work. 

\subsection{Notation}
Vectors, matrices and tensors are respectively represented in bold $\vec{v} \in \Rbb^{d_1}$, uppercase bold $\Mb \in \Rbb^{d_1\times d_2}$ and  calligraphic $\Tt \in \Rbb^{d_1\times \dots \times d_p}$. The $n$-mode product of a tensor with a \textit{vector}\footnote{Note that this notation slightly differs  from the one introduced in \cite{kolda2009tensor} where $\times_n$ denotes the $n$-mode product of a tensor with a \textit{matrix}.} 
%$\Tt \times_n \vec{v}\in\R^{d_1 \times \dots \times d_{n-1} \times d_{n+1} \times \dots \times d_p}$ 
is defined by: 
$(\Tt~\times_n~\vec{v})_{i_1,\dots,i_{n-1},i_{n+1},\dots,i_{p}}~=~\sum_{i_n=1}^{d_n} \Tt_{i_1,\dots,i_{p}}\;\vec{v}_{i_n}$. The Hadamard product, or  element-wise product, between vectors of the same size is noted $\vec u \odot \vec v$. The outer product noted $\vec u \circ \vec v$ is given by $(\vec u \circ \vec v)_{ij} = \vec u_i \vec v_j$.

\subsection{CP decomposition}
A \textit{CP decomposition} of rank $R$ factorizes a tensor $\Tt \in \Rbb^{d_1 \times d_2 \times d_3}$ into a sum of rank one tensors, $\Tt = \sum_{r=1}^R \vec{a}_r \circ \vec{b}_r \circ \vec{c}_r \equiv \CP{\A, \B, \C},$
where $\vec{a}_r \in \Rbb^{d_1}$, $\vec{b}_r \in \Rbb^{d_2}$ and $\vec{c}_r \in \Rbb^{d_3}$, and the factors matrices $\A \in \R^{d_1\times R}, \B \in \R^{d_2\times R}, \C \in \R^{d_3\times R}$ have as columns the vectors $\vec a_r,\vec b_r$ and $\vec c_r$, respectively. Using a CP decomposition reduces the number of parameters from $\mathcal O(d^3)$ to $\mathcal O(Rd)$ where $d = max\{d_1, d_2, d_3\}$.

The \textit{CP rank} of a tensor, noted $\cprank{\Tt}$, is the minimal value $R$ for which an exact CP decomposition of rank $R$ exists. 
Note that what we call a CP decomposition may not be minimal: writing a CP decomposition $\T = \CP{\A, \B, \C}$ does not imply that $\cprank{\T}=R$, but only that $\cprank{\T}\leq R$~(i.e., there may exist a smaller CP decomposition of $\T$).

\subsection{Models}
We now define the different RNN variants considered in this work and briefly explain how they relate to one another. Figure~\ref{fig:equations_summary} summarizes the set models presented.  

\begin{definition}[RNN]\label{def:RNN}
A Recurrent Neural Network $\mathcal{R}=\langle \vec{h}^0, \Ub, \Vb, \bb, \sigma \rangle$ of hidden size $n$ is parameterized by an initial hidden state $\vec{h}^0 \in \Rbb^n$,  weight matrices $\Ub \in \Rbb^{n \times d}$ and $\Vb \in \Rbb^{n \times n}$, a bias term $\bb \in \Rbb^n$ and an activation function $\sigma:\Rbb^n \rightarrow \Rbb^n$. 
Given an input sequence $(\vec{x}^1, \vec{x}^2,..., \vec{x}^T)$, a RNN computes, for each time steps $t$, the following hidden state:
$$\vec{h}^t = \sigma( \Vb\vec{h}^{t-1} + \Ub \vec{x}^t + \bb).$$
\end{definition}

Second-order RNNs extend RNNs by incorporating bilinear interactions between input and (previous) hidden state. 
\begin{definition}[2RNN]
A Second-order Recurrent Neural Network $\mathcal{A}=\langle \vec{h}^0, \At, \Ub, \Vb, \bb, \sigma \rangle$ of hidden size $n$ is parameterized by a recurrent weight tensor $\At \in \Rbb^{n \times d \times n}$ and the other terms as in Definition \ref{def:RNN}. Given an input sequence $(\vec{x}^1, \vec{x}^2,..., \vec{x}^T)$, a 2RNN computes the following hidden state for each time step $t$ :
$$\vec{h}^t = \sigma(\At \times_1 \vec{h}^{t-1}\times_2 \vec{x}^t + \Vb\vec{h}^{t-1} + \Ub \vec{x}^t + \bb).$$
\end{definition}

In the absence of the second-order term ($\At=0$), we recover the definition of a (first-order) RNN. 
When the model is restricted to the bilinear term (\textit{i.e.} $\Ub$, $\Vb$ and $\vec b$ are null), it is referred to as a BIRNN. It is worth mentioning that the class of linear BIRNNs exactly corresponds to weighted languages recognized by weighted finite automata~\cite{li2022connecting}. The next class of models, CPRNNs, parameterizes the second-order weight tensor of 2RNNs using the CP decomposition. 
\begin{definition}[CPRNN]\label{def:cprnn}
A CP Recurrent Neural Network $\mathcal{A}_{CP}=\langle \vec{h}^0, \Ab, \Bb, \Cb, \Ub, \Vb, \bb, \sigma \rangle$ with hidden size $n$ and rank $R$ is a 2RNN whose second-order term is parameterized by a CP decomposition $\CP{\Ab, \Bb, \Cb}$ with $\Ab \in \Rbb^{n \times R}$, $\Bb \in \Rbb^{d \times R}$ and $\Cb \in \Rbb^{n \times R}$. Given an input sequence $(\vec{x}^1, \vec{x}^2,..., \vec{x}^T)$, a CPRNN computes for each time steps $t$ the following hidden state:
$$\vec{h}^t = \sigma(\CP{\Ab, \Bb, \Cb} \times_1 \vec{h}^{t-1}\times_2 \vec{x}^t  + \Vb\vec{h}^{t-1} + \Ub \vec{x}^t + \bb).$$
\end{definition}
When the first-order terms are null, the model is referred to as a CPBIRNN ($\mathcal{A}_{CPBI}$). 
Note that even though the hyper-parameter $R$ is called the \emph{rank} of the CPRNN (or CPBIRNN), the CP rank of the tensor $\CP{\Ab, \Bb, \Cb}$ is only upper-bounded by $R$~(it can be strictly smaller, e.g. choosing all weights to be ones, in which case the CP rank of $\CP{\Ab, \Bb, \Cb}$ would be one for any value of $R$).  The last model we introduce is the MIRNN, which was proposed in~\cite{wu2016multiplicative} as a way to drastically reduce the number of parameters of 2RNNs. 

\begin{definition}[MIRNN]\label{def:mirnn}
A Recurrent Neural Network with Multiplicative Integration $\mathcal{A}_{MI}=\langle \vec{h}^0, \boldsymbol{\alpha},\boldsymbol{\beta_1}, \boldsymbol{\beta_2}, \Ub, \Vb, \sigma \rangle$ of hidden size $n$ is parameterized as Definition \ref{def:RNN} with additional gate vectors $\boldsymbol{\alpha},\boldsymbol{\beta_1}, \boldsymbol{\beta_2} \in \Rbb^n$. The hidden state at time step $t$ computed by a MIRNN given an input sequence $(\vec{x}^1, \vec{x}^2,..., \vec{x}^T)$ is:
$$\vec{h}^t = \sigma(\boldsymbol{\alpha} \odot \Vb \hb^{t-1} \odot \Ub \xb^{t} +\boldsymbol{\beta_1} \odot \Vb\vec{h}^{t-1} + \boldsymbol{\beta_2} \odot \Ub \vec{x}^t +  \bb).$$
\end{definition}

The activation functions considered in this work are either bijective~(e.g. linear and $\tanh$) or the Rectified Linear Unit (ReLU)~\citep{nair2010rectified} applied element-wise. We use the notation $\vec{a}^t$ to refer to the pre-activation vectors at each time step, that is the hidden state prior the activation function, $\vec{h}^t = \sigma(\vec{a}^t)$.

\begin{figure}[h]
\includegraphics[scale=0.2]{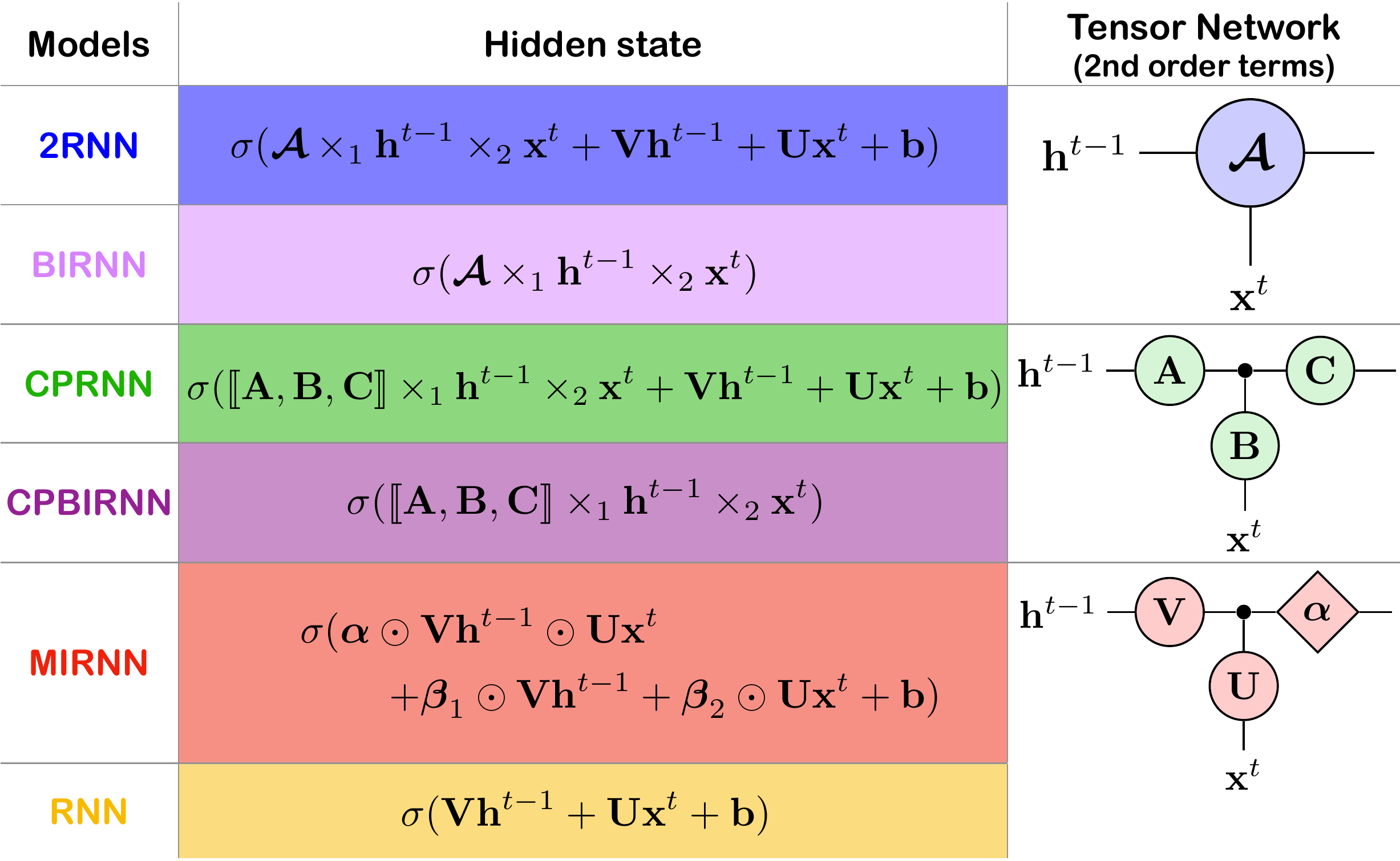}
\vspace{-0.5cm}
\caption{Different recurrent models considered in this work with their hidden state computation and the tensor network representation of their second-order term (the diamond shape represents a diagonal matrix).}\label{fig:equations_summary}
\end{figure}

\section{Theoretical results}
In order to compare the expressive power of CPRNNs for different ranks and hidden sizes~(and with other models), we formally define the set of functions they can represent. 
\begin{definition}\label{def:functionclass}
We denote by $\hCPRNN{R,n}$ the set of functions $h$ mapping input sequences to the hidden states sequences computed by CPRNNs of rank $R$ and hidden size $n$: $h(\vec{x}^1, \vec{x}^2,..., \vec{x}^T) = (\vec{h}^1, \vec{h}^2,..., \vec{h}^T)$,
where the hidden state vectors $\vec{h}^t$ are defined in Def.~\ref{def:cprnn},
\end{definition}

In the same way, we define the classes of functions $\hCPBIRNN{R,n}$, $\hRNN{n}$ and $\hMIRNN{n}$ respectively computed by CPBIRNNs, 2RNNs and MIRNNs of hidden size $n$, respectively. The theoretical findings presented in this section are summarized in Figure~\ref{fig:sumary}.

\subsection{Expressive power of CPRNNs: Relating tensor space to function space through rank}

Intuitively, the rank of a CPRNN is an hyper-parameter controlling the tradeoff between expressive power and parameter efficiency. We thus expect the expressivity to increase with rank, up to some potential saturating point. 
Let $\Rmax$ be the maximal CP rank for a family of tensors sharing the same dimensions:
$$R^{d_1, d_2 , d_3}_{\max} = \max \{\cprank{\T} \mid \T \in \R^{d_1\times d_2 \times d_3} \}.$$ 
$\Rmax$ represents the saturation point in tensor space, which translates into a saturation in expressivity of CPRNNs for ranks above that threshold. This comes from the fact that \textit{an inclusion over the space of tensors parameterizing CPRNNs directly implies an inclusion in the space of functions computed by CPRNNs}. 

Indeed, to see this, consider CPRNNs of ranks $R$ and $R+1$. The second-order terms computed by CPRNNs of rank $R$ are parameterized by CP decompositions that are less expressive than the ones from CPRNNs with rank $R+1$~(\textit{i.e.} there is inclusion in the tensor space). Looking at the computation of their hidden state $\vec h^t$, it is easy to see that any function computed by a CPRNNs $\mathcal{A}_{CP}$ of rank $R$ can be computed by a CPRNNs of rank $R+1$. It suffices to use the same weights as $\mathcal{A}_{CP}$ and to pad the extra dimensions with zeros. When $R>\Rmax$, CPRNNs of ranks $R$ and $R+1$ become equally expressive. Indeed, the expressive power of the CP decompositions parameterizing their second-order terms saturates, therefore they both can be parameterized by a minimal CP decomposition padded with zeros. 

Now, to show that a class of CPRNNs is \textit{strictly} more expressive than another one, the direct implication from the expressivity of the tensor space does not hold. Indeed, under the argument for inclusion discussed above, lies the assumption that the computation at each time step $\vec h^t$ is a function of two variables, $(\vec x^t, \vec h^{t-1}) \mapsto \vec h^t$ and that $\vec h^{t-1}$ is "free". But it is not, it is a function of $\vec h^0$ and the previous inputs. This is crucial when trying to prove \emph{strict inclusion } results: \textit{a strict inclusion of sets of tensor parameters does not imply strict inclusion of the corresponding classes of recurrent functions}. In other words, the recurrent nature of the function $h$ cannot be ignored. The relation between tensor space and function space thus becomes non-trivial when considering strict inclusion, raising the question: under which conditions does the expressive power of CPRNNs \emph{strictly} increase with rank?

To address this question we consider the notion of \textit{typical} ranks. These are CP rank values that have a non-zero probability of occurring in random tensors (\textit{i.e.} ranks $R$'s such that the set of  tensors of CP rank $R$ have positive Lebesgue measure).
Our first theoretical result states that the expressivity of CPBIRNNs increases with the rank up to the maximal \textit{typical} rank $\Rtypmax$. The result generalizes to CPRNNs with linear activation function. 
Theorem~\ref{thm:rankcprnn} also provides $\Rmax$ as a bound above which the expressive power of CPRNNs saturates. Whether the inclusion is strict or not in the gap between $\Rtypmax$ and $\Rmax$  for CPBIRNNs remains an open question.

\begin{theorem} \label{thm:rankcprnn}
The following hold for any $n$ and $d$:
\begin{itemize}[leftmargin=*]
\vspace{-0.2cm}
    \item $\hCPRNN{R,n} \subseteq \hCPRNN{R+1,n}$ for any $R$.
    \item $\hCPRNN{R, n} = \hCPRNN{R+1, n} $ for any $R\geq R_{max}$.
\end{itemize}
Moreover, assuming  $n\leq d$:  
\begin{itemize}[leftmargin=*]
\vspace{-0.2cm}
    \item $\hCPBIRNN{R,n} \subsetneq \hCPBIRNN{R+1,n}$ for any $ R < \Rtypmax$ and any real analytic invertible activation function.
    \item $\hCPRNN{R,n} \subsetneq \hCPRNN{R+1,n}$ for a linear activation function and any $ R < \Rtypmax$.
\end{itemize}
\end{theorem} 

Note that the first two points also hold for CPBIRNNs~(since they are CPRNNs restricted to second-order interactions).
We conjecture that the third and fourth points of this theorem generalize to any common activation function~(hyperbolic tangent, ReLU and linear). For the strict inclusions (third and fourth points), the assumption $n\leq d$ is made in order to preserve the emphasis of the results and the proofs on the interplay between rank and hidden size, as opposed to a limitation coming from the input size. 

\begin{figure}[h]
\includegraphics[scale=0.21]{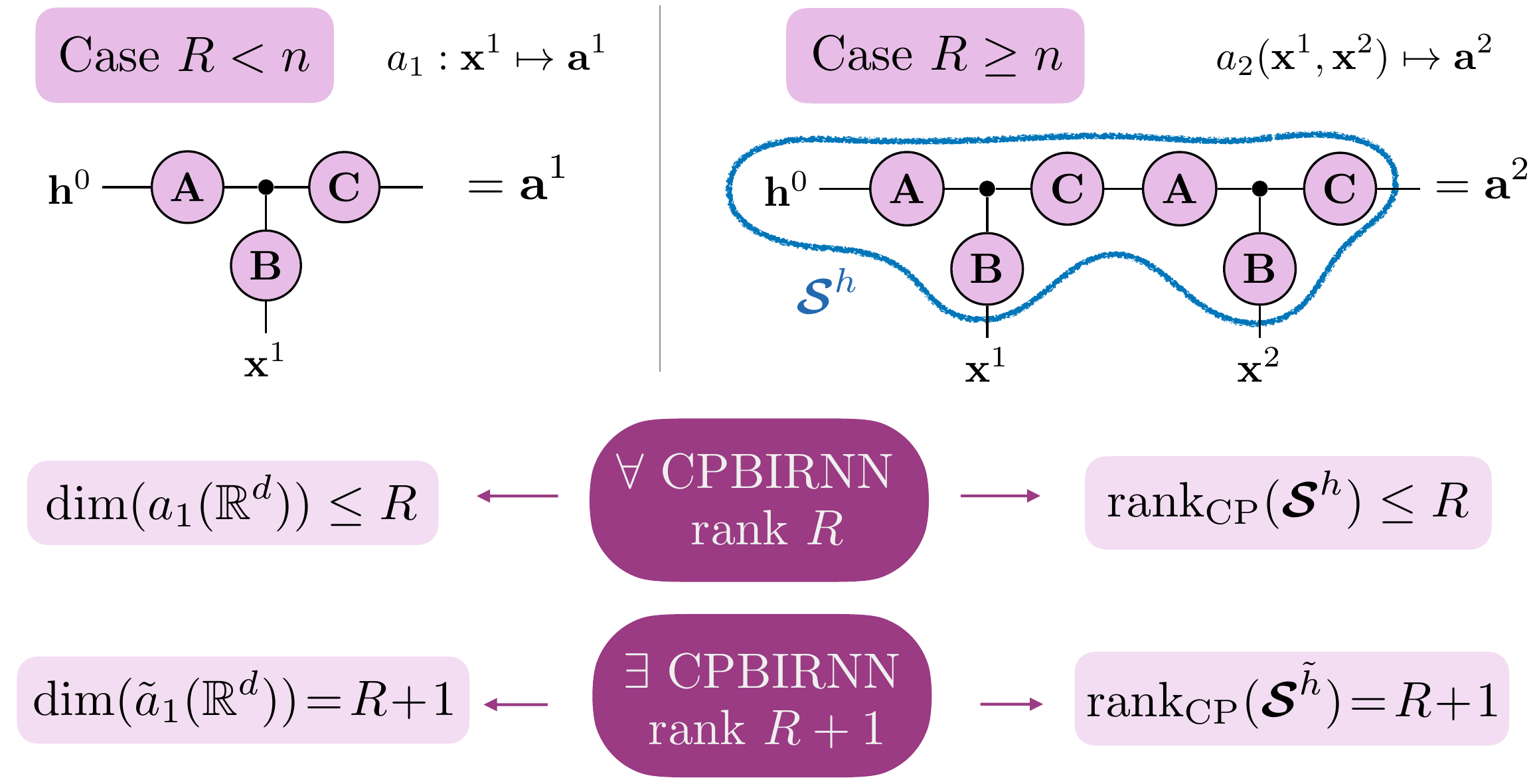}
\vspace{-0.5cm}
\caption{Elements of proof for strict inclusion of CPBIRNNs in Theorem~\ref{thm:rankcprnn}. Details of the proof can be found in the Appendix. }\label{fig:sketch_proof}
\vspace{-0.3cm}
\end{figure}

\paragraph{Sketch of proof}
The method to prove inclusions and saturation consists in finding explicit  parameterizations such that there is equality between the latent vectors $\vec h^t$'s, as outlined in the beginning of this section. 

To show strict inclusion, that is $\hCPBIRNN{R+1, n} \not \subset \hCPBIRNN{R, n}$, two cases are considered: $R<n$ and $R \geq n$. Figure~\ref{fig:sketch_proof} illustrated the key elements of these proofs. The idea behind the proof when $R<n$ is to consider the linear mapping of the first pre-activation vector $a_1:\vec{x}^1 \mapsto \vec a^1 \in \R^n$. For CPBIRNNs of rank $R$ the dimension of the image of $a_1$ is upper bounded by $R$. 
At the same time, we show that there exists a $\tilde h$ computed by a CPBIRNN of rank $R+1$ for which the pre-activation map $\tilde{a}_1$ has an image of dimension equal to $R+1$. This implies that the dimension of the manifold formed by the hidden vectors $\tilde{\vec h}^1= \sigma(\tilde{\vec a}^1)$ of this CPBIRNNs is $R+1$ (for invertible $\sigma$). This cannot be the case of any CPBIRNN of rank $R$, as this dimension is upper bounded by $R$, which concludes the case $R<n$.

In the case $R \geq n$, the rank does not act as a bottleneck on the dimension of the space of first hidden states.Therefore, we turn to the computation of the second hidden state $\vec h^2$, looking at the mapping of the pre-activation vector $a_2(\vec x^1,\vec x^2 ) \mapsto \vec a^2$. We consider the tensor $\St^h\in\R^{d\times d\times n}$ defined by $\St^h_{ijk} = [a_2(\vec{e}_i,\vec{e}_j )]_k$, where $\vec e_i$ denotes the $i$th vector of the canonical basis of $\R^d$. Intuitively, $\St^h$ gathers all the second hidden states obtained by applying the CPBIRNN to length 2 sequences of one-hot encodings and, crucially, is a witness to the low CP rank structure of the model.  

More precisely, we show that (i) the CP rank of this tensor is upper bounded by the rank of the CPBIRNN and (ii) there exists a parameterization such that this limit can be reached when the rank is smaller or equal to $\Rtypmax$. Therefore, for $R<\Rtypmax$, we conclude that there exists a function computed by a CPBIRNN of rank $R+1$ whose tensor $\St^h$ has CP rank $R+1$, and thus cannot be computed by a CPBIRNN of rank $R$.

For CPRNNs, we again look at the computation of the second hidden state vector, $h_2:(\vec x^1,\vec x^2 ) \mapsto \vec h^2$. This time however, because the activation function is linear, $h_2$ can be decomposed in four terms: $h_2(\vec x^1,\vec x^2 ) = \alpha(\vec x^1, \vec x^2 ) + \beta(\vec x^1) + \gamma(\vec x^2 ) + \delta$ where $\alpha$ is a bilinear map containing only second-order terms. This decomposition is such that given two functions $h$ and $\hat{h}$, if $\alpha \neq \hat{\alpha}$ then $h\neq \hat{h}$. Therefore, we define the tensor $\St^{\alpha}_{ijk} = [\alpha(\vec{e}_i,\vec{e}_j )]_k$ and apply a similar argument as for CPBIRNNs.
The complete proof can be found in Appendix. 

\paragraph{Explicit bounds}
Note that unlike matrices, for higher order tensors $\Rmax^{d_1, d_2 , d_3}$  is not given by the smallest dimension~$\min\{d_1,d_2,d_3\}$; it can even be greater than the largest dimension. For a third-order tensor, a loose upper-bound is given by $R^{d_1, d_2 , d_3}_{\max} \leq \min\{d_1d_2, d_1d_3, d_2d_3\}$, but the exact value is in general unknown and bounding it is a non-trivial problem~(see e.g. \citep{howell1978global}). In the context of Theorem \ref{thm:rankcprnn}, this upper bound implies that  the saturation result is valid for $R\geq \min\{nd,n^2\}$. Similarly, while the notion of \textit{typical} rank for matrices leads to a unique value that  coincides with $\Rmax$, for higher-order tensors there can be a set of typical ranks that does not necessarily include $\Rmax$. The characterization of \textit{typical} ranks is also a complex problem, but there is a known lower bound on the smallest typical rank which guarantees that the last two bullets of Theorem \ref{thm:rankcprnn} hold when $R<\frac{n^2d}{(2n+d-2)}$ ~\citep{strassen1983rank, brockett1976generic, comon2009generic}.

\subsection{Expressive power of CPRNNs: interplay between rank and hidden size}
We have established that the rank plays a key role in the expressive power of CPRNNs. The following natural question is to what extent increasing hidden size impacts the models expressivity for a fixed rank. 
Comparing the classes of functions $\hCPBIRNN{R,n}$ and $\hCPBIRNN{R,n+1}$ is, however, an ill-defined problem because their output space is different. 
It nonetheless makes sense to compare the expressiveness of these classes after being composed with another family of simple functions bringing them to a common space. We naturally consider linear maps.
Let $\linearmaps{n,q}$ denote the space of linear maps from $\Rbb^n$ to $\Rbb^q$ with $q \geq 1$. 
To formally study the impact of hidden size in CPRNNs, we compare the composition of functions : $\linearmaps{n,q} \circ \hCPBIRNN{R,n}$ which can be thought of as an output layer.
The following theorem shows that the rank of the CP decomposition acts as a bottleneck to the expressive power of CPRNNs: for hidden sizes below the CPRNN rank, the expressive power of CPRNNs improves with its hidden size; however, once the hidden size gets larger than the rank, the expressivity of the model saturates.

\begin{theorem} 
\label{thm:hiddencprnn}
The following hold for any $d$ and $n$: 
\begin{itemize}[leftmargin=*]
\vspace{-0.2cm}
    \item $\linearmaps{n,q} \circ \hCPBIRNN{R,n} \subseteq \linearmaps{n+1,q} \circ \hCPBIRNN{R,n+1}$ for any $R$ and $n$.
    \item $\linearmaps{n,q} \circ \hCPBIRNN{R,n} = \linearmaps{n+1,q} \circ \hCPBIRNN{R,n+1}$ for any $\ n \geq R$ and linear activation function.
 \end{itemize}
 Moreover, assuming  $n\leq d$:
\begin{itemize}[leftmargin=*]
\vspace{-0.2cm}
    \item $\linearmaps{n,q} \circ \hCPBIRNN{R,n} \subsetneq \linearmaps{n+1,q} \circ \hCPBIRNN{R,n+1}$ for any $n<R$ and any invertible activation function satisfying $\sigma(0)=0$.
\end{itemize}
\end{theorem}

We conjecture that the second point (saturation) and  the third point (strict inclusion) of the theorem also hold for any real analytic activation function. 
Note that this theorem is intentionally stated only for CPBIRNNs. 
%in order to focus on the interplay with the rank. In CPRNNs the difference in hidden size would also impact the expressivity of first order and bias terms.
In CPRNNs the expressivity would also be impacted by the difference in hidden size of the first-order and bias terms, while our focus is to compare the second-order terms.
\begin{figure}[h]
\includegraphics[scale=0.28]{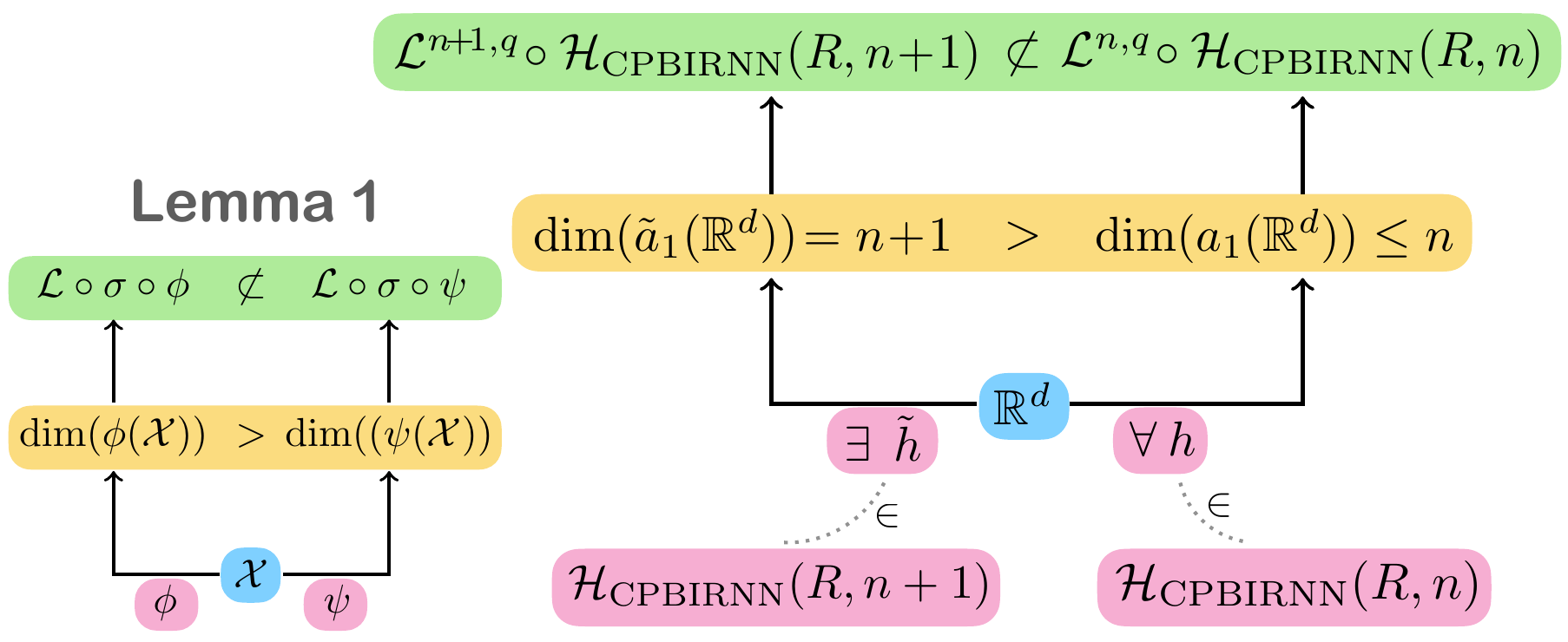}
\vspace{-0.7cm}
\caption{Schematic representation of Lemma \ref{lemma:diff.ima.dims.compose.linear.form} and its application in the context of  Theorem~\ref{thm:hiddencprnn}.}\label{fig:scheme}
\end{figure}
\paragraph{Sketch of proof}
The main technical challenge of the proof is to show the strict inclusion~(details for inclusion and saturation can be found in appendix).
We want to show that, when $n<R$, $\linearmaps{n,q} \circ \hCPBIRNN{R,n}$ is strictly included in $\linearmaps{n+1,q} \circ \hCPBIRNN{R,n+1}$. The key idea is to leverage the following two facts: 
\begin{itemize}
    \item the dimension of the linear space formed by all first hidden states pre-activation computed by a function $h\in\hCPBIRNN{R,n}$ is bounded by $n$, i.e. $\text{dim}(a_1(\Rbb^d))\leq n$,
    \item there exists a function $\tilde{h}\in\hCPBIRNN{R,n+1}$ for which the corresponding linear space is of dimension $n+1$, i.e. $\text{dim}(\tilde{a}_1(\Rbb^d))=n+1$.
\end{itemize}
While these two facts are obvious~(when $n<R$), it is not trivial to show that they imply the existence of a function $\ell \circ h\in \linearmaps{n+1,q} \circ \hCPBIRNN{R,n+1}$ that cannot be computed by any function in $\linearmaps{n,q} \circ \hCPBIRNN{R,n}$, since~(in some very loose sense) the strictly greater intrinsic linear dimension of $\tilde{h}$ is collapsed into a $q$-dimensional space through an activation function and a linear projection (note that if $q>n$ and the activation is linear, the result is somehow trivial). 

This difficulty is addressed by the following lemma showing that as long as the activation $\sigma$ is a homeomorphism, the linear bottleneck on the space of pre-activations implies a strict inclusion of the function classes after non-linear transformation through $\sigma$ and linear projection. This result is schematically illustrated in Figure~\ref{fig:scheme} as well as how it is applied to demonstrate the strict inclusion of Theorem~\ref{thm:hiddencprnn}.
\begin{lemma}\label{lemma:diff.ima.dims.compose.linear.form}
  Let $V$ be a vector space of dimension $d$, $\phi, \psi : \mathcal{X} \to V$ two maps whose images $\phi(\mathcal{X})$ and $\psi(\mathcal{X})$ are subspaces of $V$ and $\sigma $ an homeomorphism. Lastly, let $\linearmaps{}(V)$ denote the set of all linear forms on $V$~(i.e. $\linearmaps{}(V)$ is the dual space $V^*$). If $\dim(\phi(\mathcal{X})) > \dim(\psi(\mathcal{X}))$, then $\linearmaps{}(V) \circ \sigma \circ \phi \not \subset \linearmaps{}(V) \circ \sigma \circ\psi$.
\end{lemma}
Note that this lemma could be applied more broadly to other neural network architectures.

\subsection{Comparison with RNNs, 2RNNs and MIRNNs}
First, it is trivial to check that CPRNNs are strictly more expressive than RNNs since both models share the same first-order terms. In fact, one can view RNNs as CPRNNs with rank $R=0$. Similarly, one can easily check that there is no inclusion relation~(one way or the other) between RNNs and CPBIRNNs since the latter only contain second-order terms while the former only captures first-order ones. 

Turning to the comparison with 2RNNs,  Theorem \ref{thm:rankcprnn} naturally implies that (i) 2RNNs are strictly more expressive than CPBIRNNs for ranks smaller than the maximal typical rank and (ii) the expressivity of CPRNNs saturates when they become as expressive as 2RNNs, that is at $\Rmax$, leading to the following  corollary.

\begin{corollary}\label{cor:2rnn}
For any $d$ and $n$,  we have:

\begin{itemize}[leftmargin=*]
\vspace{-0.2cm}
    \item $\hCPRNN{R, n} = \hRNN{n}$ for any $ R \geq R_{max}$ (for any activation function)
\end{itemize}
Moreover, assuming $n\leq d$:
\begin{itemize}[leftmargin=*]
\vspace{-0.2cm}

    \item $\hCPBIRNN{R, n} \subsetneq \hRNN{n}$ for any $ R < \Rtypmax$ and any real analytic invertible activation function.
\end{itemize}
\end{corollary}
Note that the second point~(strict inclusion) also holds for CPRNNs with linear activation functions~(as a corollary of Theorem~\ref{thm:rankcprnn}) and the first point also holds for CPBIRNNs~(trivially). 

Another approach suggested to introduce second-order terms to RNNs while preserving a similar number of parameters is the MIRNN. Inspired by the gated mechanisms of LSTM and GRU, MIRNNs limit second-order interactions to terms that are linear in the component-wise/Hadamard product of the hidden state and input (see Def. \ref{def:mirnn}):$$ \diag(\boldsymbol{\alpha}) (\Vb \hb^{t-1} \odot \Ub \xb^{t}).$$
At the same time, the CP decomposition of the second-order term in CPRNN can also be expressed in terms of a Hadamard product:
$$\CP{\Ab,\Bb,\Cb} \times_1 \hb^{t-1} \times_2 \xb^{t}  = \C ( \A^\top \hb^{t-1} \odot \B^\top \xb^t). $$
This illustrates that MIRNNs are in fact CPRNNs of rank $n$ whose matrix $\C$ is constrained to be diagonal.
Consequently, $ \hMIRNN{n} \subseteq \hCPRNN{n,n}$.
The constraint on the matrix $\C$ to be diagonal means that each component of the second-order term of a MIRNN hidden state is computed by a rank one matrix while for a CPRNN it is rank $R$. This suggests that 
CPRNNs are strictly more expressive than MIRNNs for ranks greater than the hidden size. To prove it, it suffices to use the strict inclusion of Theorem~\ref{thm:rankcprnn} for linear CPRNNs. Indeed, in this case and when $R>n$ we have $\hMIRNN{n} \subseteq \hCPRNN{n,n} \subsetneq \hCPRNN{R,n}$.
When the rank is smaller than the hidden dimension, there is no clear inclusion relation~(in one way or the other) between the two model families. 
The following corollary~(whose proof can be found in appendix) formalizes this result.

\begin{corollary}\label{thm:mirnn}
Assuming $n\leq d$, for any $ R > n $,
\begin{itemize}[leftmargin=*]
\vspace{-0.2cm}
    \item $\hMIRNN{n}   \subseteq \hCPRNN{R,n}$ 
    \item $\hMIRNN{n}   \subsetneq \hCPRNN{R,n}$ for linear activation function
\end{itemize}
\end{corollary}

\begin{figure}[t]
\includegraphics[scale=1]{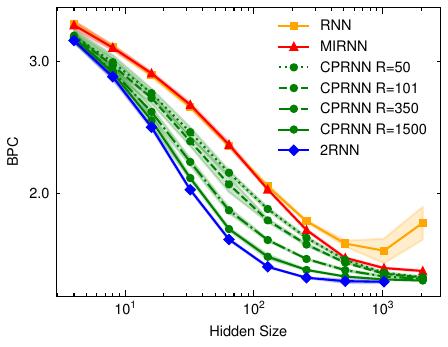}
\caption{BPC as a function of hidden size of RNN, MIRNN, 2RNN and CPRNN with ranks 50, 101, 350 and 1500. As rank increases, CPRNNs approaches 2RNNs having greater capacity for fixed hidden sizes.}\label{fig:bpc_vs_hidden}
\end{figure}

\section{Experiments}
We investigate to which extent our theoretical results hold empirically in the context of character level sequence modelling subject to gradient based training. 
We perform experiments on the Penn Treebank dataset~\citep{marcus1993building} measuring bits-per-character (BPC) using the same train/valid/test partition as in~\citet{mikolov2012subword}. All models were trained using truncated back propagation through time~\citep{werbos1990backpropagation} with sequence length of 50, batch size of 128 and using the Adam optimizer \citep{kingma2015adam} to minimize the negative log likelihood. Initial weights were drawn from a uniform random distribution~$\mathcal{U}[-\frac{1}{\sqrt{n}}, \frac{1}{\sqrt{n}}]$. For all experiments, we use the $\tanh$ activation function. For training, we use early stopping and a scheduler to reduce the learning rate (initialized at 0.001) by half on plateaus of the validation loss. All the results presented are on the test set. Each point on the plots is an average over 10 experiments run with different seeds and the variance is represented by a shaded area. 

\paragraph{Importance of hidden size} We first assess how the hidden size affects performance for RNNs, MIRNNs, 2RNNs and CPRNNs of various ranks. Figure~\ref{fig:bpc_vs_hidden} shows that the performance of all models increases with the hidden dimension. We note a slight inflexion point past a hidden size corresponding to the alphabet size, $d=101$. Interestingly, this is also where MIRNN starts to separate from RNN, showing the advantage of multiplicative interactions. As the hidden size increases, we note that the CPRNN's performance becomes closer and closer to the 2RNN's. This is also the case for rank values; as the rank increases, the CPRNN's performance gradually matches the 2RNN's. Both these observations are coherent with our theoretical results. More so, it is interesting to note that the 2RNN still outperforms all other models. This is also consistent with theory as 2RNNs encompass all other RNN models considered in this study. It should be noted that since the hidden size ($n$) directly impacts the number of parameters, the memory requirements of these models are $\mathcal O(nd+ n^2)$ for RNNs and MIRNNs, $\mathcal O(Rn+nd)$ for CPRNNs and $\mathcal O(n^2d)$ for 2RNNs. This difference in model size is also reflected in training time. Reaching early stopping with a hidden size of 512 takes on average 12mins for a RNN, 25mins to 1h for CPRNNs of ranks 1500 to 50, and 1.5h for a 2RNN.

\begin{figure}[t]
\includegraphics[scale=1]{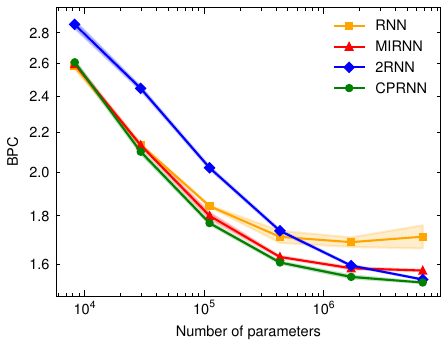}
\caption{BPC as a function of the number of parameters in RNN, MIRNN, 2RNN and CPRNN. Tuning rank and hidden size allows to find a CPRNN that outperforms all other models of same size.}
\label{fig:bpc_vs_params}
\end{figure}

\begin{figure}[t]
\includegraphics[scale=1]{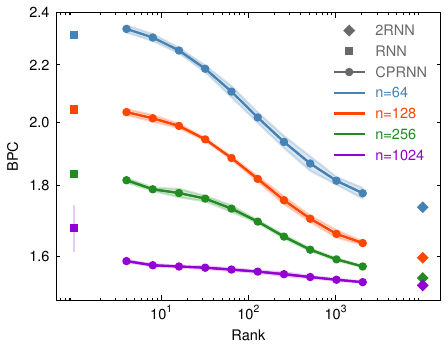}
\caption{BPC of CPRNN as a function of the rank alongside RNN and 2RNN for hidden sizes 64, 128, 256 and 1024. Variation of the rank in CPRNN interpolates between the performance of the first and second order model.}
\label{fig:bpc_vs_rank}
\end{figure}

\paragraph{Comparing models of the same size} 
The observations made in the previous paragraph could simply be explained by the respective parameter count of each model, as the number of parameters directly affects expressivity. As such, we conduct an analysis where we keep model size \textit{constant}. Figure~\ref{fig:bpc_vs_params} gives the BPC of all considered models as a function of parameter count. For CPRNNs, as many configurations of rank and hidden size can correspond to the same parameter count, we choose the best configuration (w.r.t. the validation set) for each fixed model size. The key takeaway of this figure is that it is \textit{always} possible to find a CPRNN which outperforms the other models, regardless of the number of parameters. We also notice that 2RNNs have rather poor performance below the 1M parameter threshold. This can be explained by their small hidden size values (no larger than $n=64$). Interestingly, 2RNNs match the performance of the other models at a parameter count of 1.7M, corresponding to a hidden size of 128, which is just above the input size of $d=101$. This is equally the threshold after which vanilla RNNs start overfitting, hence demonstrating the limited capacity of first-order interactions.

\paragraph{Hidden size vs. rank in CPRNNs}
Going back to the interplay between rank and hidden size, Figure~\ref{fig:bpc_vs_rank} compares the BPC of CPRNNs as a function of the rank for different hidden sizes. We include the BPC values for RNNs and 2RNNs of corresponding hidden sizes on the left and right ends of the figure, respectively. This illustrates how CPRNNs naturally interpolate between first and second-order RNNs.
We observe that for high values of hidden size ($n=1024$), the performance gain is substantial between RNNs and CPRNNs, even for very small ranks. Indeed, with a CPRNN of rank 4, which only increases the number of parameters by 2\%, the BPC already decreases by 59\% of the relative difference between RNN and 2RNN. Similarly to Figures~\ref{fig:bpc_vs_hidden} and~\ref{fig:bpc_vs_params}, we observe a slight inflection point when the rank gets greater then the input size $d=101$. 

\begin{figure}[t]
\includegraphics[scale=1]{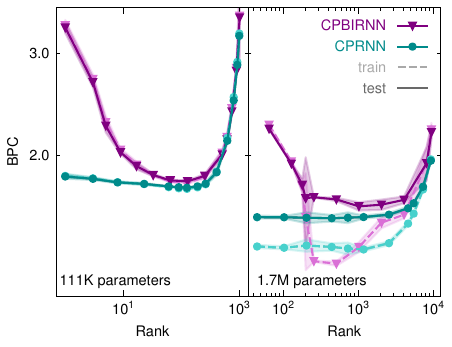}\label{fig:underfitting}
\caption{BPC of CPRNN and CPBIRNN for 111K parameters (left) and 1.7M parameters (right). CPBIRNN underfits when rank or hidden size gets low, while in CPRNN the first-order term maintain performance at low rank (i.e. high hidden size). Training and testing curves are approximately the same for 111K parameters, while they show overfitting at 1.7M of parameters.}
\vspace{-0.5cm}
\end{figure}

\paragraph{Bias-variance tradeoff in CPRNNs}
Figures~\ref{fig:bpc_vs_hidden} and \ref{fig:bpc_vs_rank} illustrate how increasing the rank and hidden size improves the capacity of CPRNNs, which is in line with our theoretical results. When it comes to comparing models at a fixed size, as done in Figure~\ref{fig:bpc_vs_params}, those two parameters are tied together and vary in opposite directions. Figure~\ref{fig:underfitting} presents the BPC of CPRNNs and CPBIRNNs as a function of rank, keeping the number of parameters fixed at 1.7M and 111K. We observe two underfitting regimes for CPBIRNNs, one corresponding to the rank being too low and the other corresponding to the hidden size being too low (which corresponds to high values of the rank). The optimal combination of rank and hidden size thus lies in between, maximizing neither of these parameters. When the number of parameters is sufficiently large (1.7M), the overfitting regime is more sensitive to the hidden size. By comparing CPBIRNNs and CPRNNs, we see the importance of first-order terms in the interplay between rank and hidden size. In CPRNNs, there is almost no underfitting for small rank. Increasing the hidden size makes the first-order term gain importance in the total number of parameters. Thus, the resulting model is similar to a RNN, preserving some expressive power, whilst without this term, the limitation of small rank is clear.

\section{Conclusion}
In this work, we formally and empirically characterize the expressivity of RNN models with various levels of second-order interactions and show how they relate to one another. In particular, we use the rank of CPRNNs as a means of interpolation between RNNs (only first-order interactions), MIRNNs (restricted second-order interactions) and 2RNNs (unrestricted second-order interactions). 
Our analysis sheds light on how of the rank and hidden size affects the expressivity of CPRNNs by providing thresholds above which it saturates. 
We determine in which measure CPRNNs are strictly more powerful than MIRNNs and when they are equivalent to 2RNNs. 
We corroborate our theoretical results with language modelling experiments on the Penn Treebank dataset.
These empirical results demonstrate that for a fixed number of parameters, we can always find CPRNNs that outperform 2RNNs and MIRNNs. This supports the hypothesis that tensor decompositions are an effective approach to optimize the bias-variance tradeoff in machine learning and motivates us to extend our study to other decompositions such as Tensor-Train \citep{oseledets2011tensor} and tensor ring \citep{zhao2016tensor}. Such study could be extended to other classes of models, including state space models and their second-order variants. Additionally, it could be applied to other types of problems, as multiplicative interactions have proven beneficial in contexts beyond sequential learning~\cite{cheng2024multilinear}.
Furthermore, exploring neural network architectures based on tensor network structures opens the possibility to introduce higher-order interactions as well as studying the impact of depth on expressivity along the lines of~\citep{cohen2016convolutional, cohen2016expressive, cohen2016inductive, sharir2017expressive, levine2018benefits, alexander2023makes, razin2024ability}. 
Another interesting avenue for future work is to explore how these model classes compare in terms of approximation instead of exact realization of the functions they compute, similarly to what was done in, e.g., \cite{cohen2016expressive}.

\section*{Impact Statement}
This paper presents work whose goal is to advance the field of Machine Learning. There are many potential societal consequences of our work, none which we feel must be specifically highlighted here.

\section*{Acknowledgement}
M. Lizaire's research is supported by the Natural Sciences and Engineering Research Council of Canada (NSERC, Vanier Scholarship) and IVADO (PhD Excellence Scholarship). G. Rabusseau's research is supported by the CIFAR AI Chair program and NSERC. In addition, we acknowledge material support from NVIDIA Corporation in the form of computational resources.

\section*{References}
\bibliography{biblio.bib}

\clearpage

\onecolumn

\appendix
\section{Proofs}
In this section, we present the proofs of Theorems \ref{thm:rankcprnn}, \ref{thm:hiddencprnn} and Corollary \ref{thm:mirnn}.

\subsection{Proof of Theorem \ref{thm:rankcprnn}}

\begin{manualthm}{\ref{thm:rankcprnn}} 
The following hold for any $n$ and $d$:
\begin{itemize}[leftmargin=*]
    \item $\hCPRNN{R,n} \subseteq \hCPRNN{R+1,n}$ for any $R$.
    \item $\hCPRNN{R, n} = \hCPRNN{R+1, n} $ for any $R\geq R_{max}$.
\end{itemize}
Moreover, assuming  $n\leq d$:  
\begin{itemize}[leftmargin=*]
    \item $\hCPBIRNN{R,n} \subsetneq \hCPBIRNN{R+1,n}$ for any $ R < \Rtypmax$ and any real analytic invertible activation function.
    \item $\hCPRNN{R,n} \subsetneq \hCPRNN{R+1,n}$ for a linear activation function and any $ R < \Rtypmax$.
\end{itemize}   
\end{manualthm}

\begin{proof}
\item
\paragraph{Inclusion for CP(BI)RNN (all bullet points)}
    The proof for the inclusion is the same whether we consider CPRNN or CPBIRNN. We show that for any $h \in \hCPRNN{R, n}$ there exists $\tilde{h} \in \hCPRNN{R+1, n}$ such that $h = \tilde{h}$ to conclude that $\hCPRNN{R,n} \subseteq \hCPRNN{R+1,n}$.
    Let $h$ be parameterized by $\mathcal{A}_{CP}=\langle \vec{h}^0, \Ab, \Bb, \Cb, \Ub, \Vb, \bb, \sigma \rangle$. 
    Now consider $\tilde{h}$ parameterized by $\tilde{\mathcal{A}}_{CP}=\langle \vec{h}^0, \tilde{\Ab}, \tilde{\Bb}, \tilde{\Cb}, \Ub, \Vb, \bb, \sigma \rangle$ using the same activation function, initial hidden state, first-order and bias terms as $\mathcal{A}_{CP}$, and padding the extra dimension of the factor matrices with $0$s, ie. $\tilde{\A}_{ij}=\A_{ij}$ and $\tilde{\A}_{i,{R+1}} = 0$ for all $i=1,\cdots,n$ and all $j \leq R$, likewise for the matrices $\tilde{\Bb}$ and $\tilde{\Cb}$. One can easily check that with these parameterizations $h=\tilde{h}$.

\item 
\textbf{Saturation (second bullet point)}
    We show that $\hCPRNN{R, n} \supseteq \hCPRNN{R+1, n}$ for all $R \geq R_{max}$. This is easy, since any $\tilde{h} \in \hCPRNN{R+1, n}$ can be computed by a CPRNN whose second-order term is parameterized by a minimal CP decomposition of rank $R_{max}$. Thus, a function $h \in \hCPRNN{R, n}$ computed by a CPRNN using this parameterization with zero padding for the extra dimensions is such that $h = \tilde{h}$. We conclude from this and the previous inclusion result ($\hCPRNN{R, n} \subseteq \hCPRNN{R+1, n}$) that $\hCPRNN{R, n} = \hCPRNN{R+1, n} \quad \forall R \geq R_{max}$.
\item
\paragraph{Strict inclusion for CPBIRNN (third bullet point)}
    The fact that the inclusion is strict for $R < \Rtypmax$, i.e. $\hCPBIRNN{R+1, n} \not \subset \hCPBIRNN{R, n}$, is the key technical element of the proof of this theorem. \\
    We begin with the easiest case $R < n$, where we can show that not all maps $\vec x^1 \mapsto \vec h^1$ computed by CPBIRNNs of rank $R+1$ can be computed by CPBIRNNs of rank $R$.
    Consider the linear mapping leading to the first pre-activation vector, $a_1:\vec{x}^1 \mapsto \vec a^1 \in \R^n$. One can easily check that:
    \begin{enumerate}[label=(\roman*)]
        \item the dimension of the image of $a_1$ is upper-bounded by $R$ for CPBIRNNs of rank $R$, i.e. $\text{dim}(a_1(\R^d))\leq R$,
        \item as long as $R<n$, there exist parameters $\A,\B,\C, \vec h^0$ of a CPBIRRN of rank $R+1$ such that the image $a_1(\R^d)$ has dimension exactly equal to $R+1$ (e.g. taking $\diag(\A^\top \vec{h}^0),\B,\C$ full rank), and
        \item the dimension of the manifold formed by the hidden vectors $\vec h^1 = \sigma(\vec a^1)$ is the same as $a_1(\R^d)$ for any invertible activation functions, since they are homeomorphisms.
    \end{enumerate}
    Therefore, any function computed by a CPBIRNN of rank $R+1$ for which the manifold of $\sigma(a_1(\R^d))$ is of dimension $R+1$ can not be computed by a CPBIRNN of rank $R$.
    This concludes the case $R < n \, (\leq \Rtypmax)$: $\hCPBIRNN{R+1, n} \not \subset \hCPBIRNN{R, n}$ for any invertible activation function.
    
    We now turn to the case $R \geq n$. Here an argument over the dimension of the first hidden state space can not be made as it is limited by the hidden size $n$ and not the rank $R$. Therefore, we turn to the computation of the second hidden state $\vec h^2$, looking at the mapping of the second pre-activation vector $a_2(\vec x^1,\vec x^2 ) \mapsto \vec a^2$ ($\vec h^2 = \sigma(\vec a^2)$). 
    For any $h \in \mathcal{H}_{\mathrm{CPBIRNN}}(R,n)$ parameterized by $\mathcal{A}_{CPBI}=\langle \vec{h}^0, \Ab, \Bb, \Cb, \sigma \rangle$, consider the tensor defined by $\St^h_{ijk} = [a_2(\vec{e}_i,\vec{e}_j )]_k$. For intuition, when the activation function is linear, this tensor computes the second hidden state vector $\vec h^2 = \St^h \times_1 \vec x^1 \times_2 \vec x^2$. 
    
    First, one can check that $\St^h = \CP{\sigma(\B \diag(\A^\top \vec{h}^0)\C^\top )\A, \B, \C}$~(where $\sigma$ is applied component-wise) and therefore $$\cprank{\St^h} \leq \cprank{\CP{\A, \B, \C}} \leq R.$$
    
    Second, we show that there exists $\tilde{h} \in \Ht_{\mathrm{CPBIRNN}}(R+1,n)$ parameterized by $\tilde{\mathcal{A}}_{CPBI}=\langle \tilde{\vec{h}}^0, \tilde{\Ab}, \tilde{\Bb}, \tilde{\Cb}, \sigma \rangle$ such that $\cprank{\St^{\tilde{h}}} = R+1$. This part of the proof relies on the following lemma, whose proof is in Appendix \ref{proof:lemma1} and leverages the fact that the probability of drawing $\A,\B,\C$~(from a distribution which is continuous w.r.t. the Lebesgues measure) such that $\cprank{\CP{\A,\B,\C}}=\Rtypmax$ is strictly positive.

\newcommand{\lemmaexistenceCPrankR}{
For all $ n\leq d \leq r \leq\Rtypmax(n,d,n)$, there exist matrices $\mat A, \mat C\in\R^{n\times r}, \mat B \in \R^{d \times r}$ and a vector $\vec h_0\in\R^n$ such that $\cprank{\CP{\mat A,\mat B,\mat C}} = r $ and $\rank(\B\diag(\A^\top\vec h_0)\C^\top) = \rank(\tanh(\B\diag(\A^\top\vec h_0)\C^\top)) = n$ (where $\tanh$ is applied component-wise).}
\begin{lemma} \label{lemma:existence.CP.rank.R}
\lemmaexistenceCPrankR
\end{lemma}

    It follows from this lemma that there exist $\tilde{\A} \in \Rbb^{n \times (R+1)}$, $\tilde{\B} \in \Rbb^{d \times (R+1)}$, $\tilde{\C} \in \Rbb^{n \times (R+1)}$ and $\tilde{\vec h}_0 \in \R^n$ such that  $\cprank{\CP{\tilde{\A},\tilde{\B},\tilde{\C}}}=R+1$ and $\rank(\tilde\B\diag(\tilde\A^\top\tilde{\vec h}_0)\tilde\C^\top)=\rank(\tanh(\tilde\B\diag(\tilde\A^\top\tilde{\vec h}_0)\tilde\C^\top))=n$. We thus have that  $\St^{\tilde{h}} = \CP{\mat M\tilde{\A}, \tilde{\B},\tilde{\C}}$ where $\mat M\in \R^{d \times n}$ is a left invertible matrix~(both when $\sigma$ is the hyperbolic tangent and the identity), which in turn implies that~(see Lemma~\ref{lem:cprank.invariance.leftinvertiblematrix})
    $$\cprank{\St^{\tilde{h}}}= \cprank{\CP{\mat M\tilde{\A},\tilde{\B},\tilde{\C}}}= \cprank{\CP{\tilde{\A},\tilde{\B},\tilde{\C}}}=R+1.$$
    We conclude that $\tilde{h} \notin \mathcal{H}_{\text{CPBIRNN}}(R,n)$ since  $\cprank{\St^{\tilde{h}}} = R+1$ while $\forall h \in \Ht_{\text{CPBIRNN}}(R,n)$ we have $ \cprank{\St^h} \leq R$. 
\item
\paragraph{Strict inclusion for linear CPRNNs (fourth bullet point)}
    In the case of linear CPRNNs, we again consider the function computing the second hidden state, but this time we only look at the term with second-order interactions between $\vec x^1$ and $\vec x^2$ and use a similar argument to the one used for the strict inclusion of CPBIRNNs. The argument is however slightly simpler in this case and can be used for both cases $R< n$ and $R\geq n$. 
    
    For any $h \in \mathcal{H}_{\mathrm{CPRNN}}(R,n)$ parameterized by $\mathcal{A}_{CP}=\langle \vec{h}^0, \Ab, \Bb, \Cb, \Ub, \Vb, \vec b, \sigma=\I \rangle$, let $h_2:(\vec x^1,\vec x^2 ) \mapsto \vec h^2$ be the associated function mapping the two first inputs to the second hidden state. 
    First observe that since $\sigma=\I$, the associated $h_2$ can be decomposed in four terms: $h_2(\vec x^1,\vec x^2 ) = \alpha(\vec x^1, \vec x^2 ) + \beta(\vec x^1) + \gamma(\vec x^2 ) + \delta$ where $\alpha$ is the bilinear map containing only the second order terms~(i.e. $\alpha$ is a linear map of the Kronecker product of $\vec x^1$ and $\vec x^2$). Note that this decomposition is such that given two functions $h,\hat{h}$, if $\alpha \neq \hat{\alpha}$ then $h\neq \hat{h}$.
    One can check that    
    $$
    \alpha(\vec x^1,\vec x^2 )= \CP{(\U^\top + \B \diag(\A^\top \vec{h}^0)\C^\top )\A, \B, \C} \times_1 \x_1 \times_2 \x_2
    $$
    It follows that defining the tensor  $\St^{\alpha}\in \R^{d\times d \times n}$  by $\St^{\alpha}_{ijk} = [\alpha(\vec{e}_i,\vec{e}_j )]_k$, we have
     $$\cprank{\St^{\alpha}} \leq \cprank{\CP{\A, \B, \C}} \leq R.$$ 
    Now, since $R <\Rtypmax$, there exist parameters of a function $\tilde{h} \in \hCPRNN{R+1,n}$ such that $\cprank{ \CP{\tilde{\A}, \tilde{\B}, \tilde{\C}}}  = R+1$. Moreover, since $n\leq d$, $\tilde{\mat U}\in\R^{n\times d}$ can always be chosen such that $ (\tilde{\U}^\top + \tilde{\B} \diag(\tilde{\A}^\top \tilde{\vec{h}}^0)\tilde{\C}^\top )\in\R^{d\times n}$ is left invertible, in which case~(see again Lemma~\ref{lem:cprank.invariance.leftinvertiblematrix})
    $$\cprank{\St^{\tilde\alpha}}  = \cprank{ \CP{(\tilde{\U}^\top + \tilde{\B} \diag(\tilde{\A}^\top \tilde{\vec{h}}^0)\tilde{\C}^\top ))\tilde{\A}, \tilde{\B}, \tilde{\C}} } = \cprank{ \CP{\tilde{\A}, \tilde{\B}, \tilde{\C}}}  = R+1.$$
      By construction we thus have $\cprank{\St^{\tilde\alpha}} = R+1$ while $\cprank{\St^\alpha} \leq R$ for all $h \in \hCPRNN{R,n}$, hence $\tilde{h} \not \in \hCPRNN{R,n}$.

\end{proof}

\subsubsection{Proof of Lemma \ref{lemma:existence.CP.rank.R}}\label{proof:lemma1}

\newcommand{\lemmaexistenceCPrankR}{
For all $ n\leq d \leq r \leq\Rtypmax(n,d,n)$, there exist matrices $\mat A, \mat C\in\R^{n\times r}, \mat B \in \R^{d \times r}$ and a vector $\vec h_0\in\R^n$ such that $\cprank{\CP{\mat A,\mat B,\mat C}} = r $ and $\rank(\B\diag(\A^\top\vec h_0)\C^\top) = \rank(\tanh(\B\diag(\A^\top\vec h_0)\C^\top)) = n$ (where $\tanh$ is applied component-wise).}
\begin{manuallemma}{\ref{lemma:existence.CP.rank.R}}
\lemmaexistenceCPrankR
\end{manuallemma}

The proof of the lemma relies on the following result whose proof can be found in \citep{mityagin2020zero}\footnote{Mityagin, Boris Samuilovich. "The zero set of a real analytic function." Mathematical Notes 107.3-4 (2020): 529-530.}. 

\begin{lemma}\label{lemma:realanalytic}
Let $A(x)$ be a real analytic function on (a connected open domain
$U$ of) $\R^d$. If $A$ is not identically zero, then its zero set
$$ Z (A) := \{x \in U \mid A(x) = 0\}$$
has a zero measure.
\end{lemma}

\begin{proof}[Proof of Lemma \ref{lemma:existence.CP.rank.R}]
Let $\A,\B,\C,\vec h_0$ be randomly drawn from a distribution which is continuous w.r.t. the Lebesgue measure. On the one hand, since $r \leq\Rtypmax(n,d,n)$, the measure of the set of matrices $\A,\B,\C$ such that $\cprank{\CP{\A,\B,\C}} =r$ is stricly positive, and thus the probability that the random matrices $\A,\B,\C$ are such that $\cprank{\CP{\A,\B,\C}} =r$ is strictly greater than $0$. On the other hand, consider the real analytic function 
$$F(\mat A,\mat B, \mat C, \vec h_0) = \mathrm{det}[(\B\diag(\A^\top\vec h_0)\C^\top)_{1:n,:}] \cdot  \mathrm{det}[\tanh(\B\diag(\A^\top\vec h_0)\C^\top)_{1:n,:}] $$
where the notation $\mat M_{1:n,:}$ denotes the matrix containing the first $n$ rows of $\mat M$.
One can check that by letting $\tilde \A,\tilde\B,\tilde\C$ be identity matrices padded with 0's and $\tilde{ \vec h}_0$ be a vector full of ones, we have $F(\tilde\A,\tilde\B, \tilde\C,\tilde{ \vec h}_0) \neq 0$, and thus $F$ is not identically zero. 
It then follows from Lemma~\ref{lemma:realanalytic} that the set of $\A,\B,\C,\vec h_0$ such that $F(\mat A,\mat B, \mat C, \vec h_0)=0$ has measure 0. Hence, with probability one, both $\B\diag(\A^\top\vec h_0)\C^\top$ and $\tanh(\B\diag(\A^\top\vec h_0)\C^\top)$ are full rank, i.e. of rank $n$~(since $n\leq d \leq r$). 
Combining these two observations lead to the existence of matrices $\A,\B,\C$ and a vector $\vec h_0$ satisfying the claim of the lemma.
\end{proof}

\subsection{Proof of Theorem \ref{thm:hiddencprnn}}

\begin{manualthm}{\ref{thm:hiddencprnn}} 
    The following hold for any $d$ and $n$: 
\begin{itemize}[leftmargin=*]
    \item $\linearmaps{n,q} \circ \hCPBIRNN{R,n} \subseteq \linearmaps{n+1,q} \circ \hCPBIRNN{R,n+1}$ for any $R$ and $n$.
    \item $\linearmaps{n,q} \circ \hCPBIRNN{R,n} = \linearmaps{n+1,q} \circ \hCPBIRNN{R,n+1}$ for any $\ n \geq R$ and linear activation function.
 \end{itemize}
 Moreover, assuming  $n\leq d$:
\begin{itemize}[leftmargin=*]
    \item $\linearmaps{n,q} \circ \hCPBIRNN{R,n} \subsetneq \linearmaps{n+1,q} \circ \hCPBIRNN{R,n+1}$ for any $n<R$ and any invertible activation function satisfying $\sigma(0)=0$.
\end{itemize}
\end{manualthm}

\begin{proof}
\item
\paragraph{Inclusion (all bullet points)}  
    We first show that  $\linearmaps{n,q} \circ \hCPBIRNN{R,n}\subseteq \linearmaps{n+1,q} \circ \hCPBIRNN{R,n+1}$ (both when $n<R$ and $n \geq R$).
    Consider a function $l \circ h \in \linearmaps{n,q} \circ \hCPBIRNN{R,n}$ with $l$ defined by $l(\vec v) = \mat W \vec v$ for some $\mat W$ in $\Rbb^{q \times n}$ and $h$ parameterized by the CPBIRNN $\mathcal{A}_{CPBI}=\langle \vec{h}^0, \Ab, \Bb, \Cb, \sigma \rangle$. 
    Now consider the function $\tilde l \circ \tilde h$ in $\linearmaps{n+1} \circ \hCPBIRNN{R,n+1}$ with 
    $\tilde l$ defined by $\tilde{l}(\vec v) = \tilde{\mat W} \vec v$ where $\tilde{\mat W}^\top= \left(\mat W^\top \atop 0 \cdots 0\right)$ and $\tilde h$ parameterized by $\tilde{\mathcal{A}}_{CPBI}=\langle \tilde{\vec{h}}^0, \tilde{\Ab}, \tilde{\Bb}, \tilde{\Cb}, \sigma \rangle$ where $\tilde{\mat A} = \left(\mat A \atop 0 \cdots 0\right)$, $\tilde{\mat C} = \left(\mat C \atop 0 \cdots 0\right) $ and $\tilde{\vec h}_0 = \left(\vec h_0 \atop 0\right)$. 
    One can check that $\tilde l \circ \tilde h = l \circ h$. This construction thus shows that for any function in $\linearmaps{n} \circ \hCPBIRNN{R,n}$ we can find an equivalent function in $\linearmaps{n+1} \circ \hCPBIRNN{R,n+1}$, i.e. $\linearmaps{n,q} \circ \hCPBIRNN{R,n} \subseteq \linearmaps{n+1,q} \circ \hCPBIRNN{R,n+1}$.
\item
\paragraph{Saturation (second bullet point)} 
    Since we already showed that $\linearmaps{n,q}  \circ \hCPBIRNN{R,n} \subseteq \linearmaps{n+1,q} \circ \hCPBIRNN{R,n+1}$ for any invertible activation function (including linear ones), we just need to show that 
    $\linearmaps{n+1,q} \circ \hCPBIRNN{R,n+1} \subseteq \linearmaps{n,q} \circ \hCPBIRNN{R,n}$. Let $\tilde{h} \in \hCPBIRNN{R,n+1}$ parameterized by $\tilde{\mathcal{A}}_{CPBI}=\langle \tilde{\vec{h}}^0, \tilde{\A}, \tilde{\B}, \tilde{\C}, \sigma=\mat I \rangle$  and $\tilde{l}:\vec v \mapsto \tilde{\mat W} \vec v \in \linearmaps{n+1,q}$. 
    Since $R<n+1$, we can factorize $\tilde{\A}$ and $\tilde{\C}$ into $\tilde{\A} = \Pb\A$ and 
    $\tilde{\C} = \C \mat Q$ where  
    $\A, \C \in\R^{n\times R}$ and $\mat P, \mat Q \in \R^{n+1\times n}$ are rank $n$.
    The functions $h \in \hCPBIRNN{R,n}$ parameterized by $\mathcal{A}_{CPBI}=\langle \vec{h}^0 = \mat P^{\top}\tilde{\vec{h}}^0, \A, \tilde{\B},\C, \sigma=\mat I \rangle$ and $l:\vec v \mapsto \mat Q^{\top}\tilde{\mat W} \vec v \in \linearmaps{n,q}$ are such that $\tilde{l} \circ \tilde{h} = l \circ h$ which shows that $\linearmaps{n+1,q} \circ \hCPBIRNN{R,n+1} \subseteq \linearmaps{n,q} \circ \hCPBIRNN{R,n}$. Note that without the restriction of a linear activation function is key here as it is what allows the factorization matrix $\mat Q^{\top}$ to be introduced in the parameterization of $l$ ($\mat Q^{\top}\tilde{\mat W}$). 
\item 
\paragraph{Strict inclusion (third bullet point)} 
    We now show that the inclusion is strict for $n<R$, i.e., that
    $$\linearmaps{n+1,q} \circ \hCPBIRNN{R,n+1} \not \subset \linearmaps{n,q} \circ \hCPBIRNN{R,n}.$$
    
    It suffices to show the result for $q=1$ which we will do. Indeed, suppose the result is true for $q=1$ and let $\tilde{l} \circ \tilde{h} \in \linearmaps{n+1,1} \circ  \hCPBIRNN{R,n+1}$ be s.t. $\tilde{l} \circ \tilde{h} \neq l \circ h$ for all $l \circ h \in \linearmaps{n,1}\circ \hCPBIRNN{R,n}$. It is easy to see that, for any $q\geq 1$, the map $\bar{l}:\vec v \mapsto (\tilde{l}(\vec v)\ 0\ \cdots\ 0)^\top \in \linearmaps{n+1,q}$ is such that $\bar{l} \circ \tilde{h} \not \in \linearmaps{n, q} \circ \hCPBIRNN{R,n}$. Then the proof goes as follow:
    \begin{enumerate}[label=(\roman*)]
        \item First, observe that for any $h \in \hCPBIRNN{R,n}$  parameterized by $\mathcal{A}_{CPBI}=\langle \vec{h}^0, \Ab, \Bb, \Cb, \sigma \rangle$, the image of the first pre-activation mapping   $a_1: \vec x^1 \mapsto \vec a^1 = \CP{\mat A,\mat B,\mat C} \times_1 \vec h^0 \times_2 \vec x^1$ is a linear space of dimension at most $n$, i.e. $\dim(a_1(\R^d))\leq n$.
        \item Second, consider $\tilde{h} \in \hCPBIRNN{R,n+1}$ parameterized by $\tilde{\mathcal{A}}_{CPBI}=\langle \tilde{\vec{h}}^0, \tilde{\Ab}, \tilde{\Bb}, \tilde{\Cb}, \sigma \rangle$ such that $\diag(\tilde{\A}^\top \tilde{\vec h}_0)$, $\tilde{\mat B}$ and $\tilde{\mat C}$ are full rank matrices (e.g. $\tilde{\A}_{1:}$ is a row of 1's and $\tilde{\vec{h}}^0= \vec{e}_1$). Since $n < R \leq d$, one can check that the image of $\tilde{a}_1: \vec x^1\mapsto \tilde{\vec a}^1$ has dimension $n+1$, i.e. $\dim(\tilde{a}_1(\R^d))=n+1$.
        \item Third, let $z: \vec v \mapsto (\vec v \; 0)^\top$ be a mapping that augments the dimension while leaving the original transformation intact.     
    We have that $z\circ a_1(\R^d)$ and $\tilde{a}_1(\R^d)$ are both linear subspaces of $\R^{n+1}$ and that $\dim(z\circ a_1(\R^d))<\dim(\tilde{a}_1(\R^d))$. 
    \end{enumerate}

    It thus follows from Lemma~\ref{lemma:diff.ima.dims.compose.linear.form}  that $\linearmaps{n+1,1} \circ \sigma \circ \tilde{a}_1 \not \subset \linearmaps{n+1,1} \circ \sigma \circ z\circ a_1 $. Since for any activation function such that $\sigma(0) =0$ we have  $\linearmaps{n,1} \circ \sigma  = \linearmaps{n+1,1} \circ \sigma \circ z$, we conclude that $\linearmaps{n+1,1} \circ \sigma \circ \tilde{a}_1 \not \subset \linearmaps{n,1} \circ \sigma \circ a_1 $. We thus have shown that there exists $\tilde{h} \in \hCPBIRNN{R,n+1}$ such that for any $h \in \hCPBIRNN{R,n}$ we have $\linearmaps{n+1,1} \circ \tilde{h}_1 \not \subset \linearmaps{n,1} \circ h_1 $ (where $h_1$ and $\tilde{h}_1$ denote the function mapping the first input vector to the first hidden state of the respective CPRNNs), which concludes the proof. 

\end{proof}

\subsubsection{Proof of Lemma \ref{lemma:diff.ima.dims.compose.linear.form}}
\begin{manuallemma}{\ref{lemma:diff.ima.dims.compose.linear.form}}
Let $V$ be a vector space of dimension $d$,   $\phi, \psi : \mathcal{X} \to V$  two maps whose images $\phi(\mathcal{X})$ and  $\psi(\mathcal{X})$ are subspaces of $V$ and $\sigma $ an homeomorphism. Lastly, let $\linearmaps{}(V)$ denote the set of all linear forms on $V$~(i.e. $\linearmaps{}(V)$ is the dual space $V^*$).
 
If $\dim(\phi(\mathcal{X})) > \dim(\psi(\mathcal{X}))$, then $\linearmaps{}(V) \circ \sigma \circ \phi \not \subset \linearmaps{}(V) \circ \sigma \circ\psi$.
\end{manuallemma}

\begin{proof}
We suppose $\linearmaps{}(V) \circ \sigma \circ \phi \subseteq \linearmaps{}(V) \circ \sigma \circ\psi$ and show it leads to a contradiction. 
Let $l_1, \dots l_d$ be $d$ linear forms from $\linearmaps{}(V)$ that are linearly independent. By hypothesis, there exist $\tilde{l}_1, \dots \tilde{l}_d$ such that $l_i \circ \sigma \circ \phi = \tilde{l}_i \circ \sigma \circ \psi$ for all $i=1,\dots,d$. Consider the mappings $\nu, \tilde{\nu}: V \to V$ given by $\nu: u \mapsto \sum_{i=1}^d l_i(u) v_i$ and $\tilde{\nu}: u \mapsto \sum_{i=1}^d \tilde{l}_i(u) v_i$ where $v_1, \dots, v_d$ is an arbitrary basis of $V$. The previous equality ($l_i \circ \sigma \circ \phi = \tilde{l}_i \circ \sigma \circ \psi$) implies that $\nu \circ \sigma \circ \phi = \tilde{\nu} \circ \sigma \circ \psi$.
On the one hand, since $\nu$ is invertible by construction, $\nu \circ \sigma$ is an homeomorphism and the dimension of the manifold (noted $\dim_{\mathcal{M}}$) of the image of $\nu \circ \sigma \circ \phi$ is the same as the dimension of the image of $\phi$, i.e. $$\dim_{\mathcal{M}}(\nu \circ \sigma \circ \psi (\mathcal{X})) = \dim(\phi (\mathcal{X})).$$
On the other hand, since $\tilde{\nu}$ is a linear transformation it can not increase the dimension of the manifold, thus $$\dim_{\mathcal{M}}(\tilde{\nu} \circ \sigma \circ \psi(\mathcal{X})) \leq \dim(\psi(\mathcal{X}))<\dim(\phi(\mathcal{X})).$$
This is a contradiction since the hypothesis led to $\nu \circ \sigma \circ \phi = \tilde{\nu} \circ \sigma \circ \psi$ which implies $\dim_{\mathcal{M}}((\tilde{\nu} \circ \sigma \circ \psi)(\mathcal{X})) = \dim_{\mathcal{M}}((\tilde{\nu} \circ \sigma \circ \phi)(\mathcal{X}))$.
\end{proof}

\subsection{Proof of Corollary~\ref{thm:mirnn}}
\begin{manualcor}{\ref{thm:mirnn}} 
Assuming $n\leq d$, for any $ R > n $,
\begin{itemize}[leftmargin=*]
    \item $\hMIRNN{n}   \subseteq \hCPRNN{R,n}$ 
    \item $\hMIRNN{n}   \subsetneq \hCPRNN{R,n}$ for linear activation function
\end{itemize}
\end{manualcor}

\begin{proof}
\item
\paragraph{Inclusion} 
    We show that for any $\hMIRNN{n}\subseteq \hCPRNN{R,n}$ for $R \geq n$. Note the inclusion is valid for CPRNNs and without restrictions on the first order and biais terms of MIRNNs.
    First observe that an equality between the second-order terms of a CPRNN and a MIRNN corresponds to
    $\sum_{r=1}^R \Ab_{ir} \Bb_{jr} \Cb_{kr} = \boldsymbol{\alpha}_k\Vb_{ki} \Ub_{kj}$
    for all $i,k=1,\cdots,n$ and $j=1,\cdots, d$.
    Thus, for any $\bar{h} \in \hMIRNN{n}$ parameterized by $\mathcal{A}_{MI}=\langle \vec{h}^0, \bar{\Ub}, \bar{\Vb}, \boldsymbol{\alpha}, \boldsymbol{\beta_1}, \boldsymbol{\beta_2}, \bb, \sigma \rangle$ a CPRNN given by $\mathcal{A}_{CP}=\langle \vec{h}^0, \Ab, \Bb, \Cb, \Ub = \boldsymbol{\beta_1 \odot \bar{\Ub}}, \Vb = \boldsymbol{\beta_2 \odot \bar{\Vb}}, \bb, \sigma \rangle$ with 
    $$
    \A_{ir}=
    \begin{cases}
    \bar{\Vb}_{ir} & \text{if } r,i=1\dots n \\
    \boldsymbol{0} & \text{else}
    \end{cases}
    \quad
    \B_{jr}=
    \begin{cases}
    \bar{\Ub}_{jr} & \text{if } r = 1\dots n \; j=1\dots d \\
    0 & \text{else}
    \end{cases} 
    \C_{kr}=
    \begin{cases}
    \boldsymbol{\alpha}_{k} & \text{if } k=r \\
    0 & \text{else}
    \end{cases} 
    $$ 
    computes a function $h$ that is equal to $\bar{h}$.
\item 
\paragraph{Strict inclusion} 
The strict inclusion comes from the observation that since we have $\hMIRNN{n}\subseteq \hCPRNN{R,n}$. As Theorem~\ref{thm:rankcprnn} states that for linear activation function and $R<\Rmax$ 
$\hCPRNN{R,n}\subsetneq \hCPRNN{R+1,n}$, it follows that for any $R>n$ and a linear activation function $\hMIRNN{n}\subsetneq \hCPRNN{R,n}$.
\end{proof}

\subsection{Additional Lemmas}

\begin{lemma}\label{lem:cprank.invariance.leftinvertiblematrix}
    Let $\T$  be a tensor of CP rank $R$ and let $\T = \CP{\A,\B,\C}$ be a minimal (i.e. rank $R$) CP decomposition of $\T$. Then, for any left invertible matrix $\M$, $\cprank{\T} = \cprank{\CP{\M\A,\B,\C}}$.
\end{lemma}
\begin{proof}
    Let $\vec a_1,\cdots, \vec a_R$ be the columns of $\A$, $\vec b_1,\cdots, \vec b_R$ the columns of $\B$ and $\vec c_1,\cdots, \vec c_R$ be the columns of $\C$. 

    First, it is trivial to show that $R \geq \cprank{\CP{\M\A,\B,\C}}$ since $\CP{\M\A,\B,\C}$ is a rank $R$ CP decomposition. To show the result, we need to show that this CP decomposition is minimal. Suppose it is not the case, i.e., there exists a rank $S<R$ CP decomposition $\CP{\M\A,\B,\C} = \CP{\tilde\A, \tilde\B, \tilde \C}$ where $\tilde\A, \tilde\B, \tilde \C$ all have $S$ columns. Then, one can easily check that  $\T = \CP{\M^+\tilde\A, \tilde\B, \tilde \C}$, where $\M^+$ is the left inverse of $\M$, i.e. $\T$ admits a CP decomposition of rank $S < R$, a contradiction. 
    
\end{proof}
\end{document}